\documentclass{article} %
\usepackage{iclr2025_conference,times}
\usepackage[utf8]{inputenc} %
\usepackage[T1]{fontenc}    %
\usepackage{hyperref}       %
\usepackage{url}            %
\usepackage{booktabs}       %
\usepackage{amsfonts}       %
\usepackage{nicefrac}       %
\usepackage{microtype}      %
\usepackage{xcolor}  
\usepackage{amsmath}
\usepackage{amssymb}
\usepackage{lscape}
\usepackage{mathtools}
\usepackage{algorithm}
\usepackage{algpseudocode}
\usepackage{wrapfig}
\usepackage{amsthm}%
\usepackage{xspace}
\usepackage{multirow}
\usepackage{tikz}
\usepackage{makecell}
\usepackage{enumitem}

\newcommand*\circled[1]{\tikz[baseline=(char.base)]{
            \node[shape=circle,draw,inner sep=0.6pt] (char) {#1};}}
\usepackage[capitalize,noabbrev]{cleveref}

\usepackage{amsmath,amsfonts,bm}

\def\eqref#1{equation~\ref{#1}}

\def\1{\bm{1}}

\DeclareMathAlphabet{\mathsfit}{\encodingdefault}{\sfdefault}{m}{sl}
\SetMathAlphabet{\mathsfit}{bold}{\encodingdefault}{\sfdefault}{bx}{n}

\DeclareMathOperator*{\argmax}{arg\,max}

\newcommand{\ours}{BFPO}
\newcommand{\oursfull}{Bi-Factorial Preference Optimization}

\newcommand{\Is}{I_\text{safe}}
\newcommand{\Ih}{I_\text{help}}
\newcommand{\ps}{p_\text{safe}^*}
\newcommand{\ph}{p_\text{help}^*}

\newcommand{\supp}{\text{Supp}}

\makeatletter
\DeclareRobustCommand\onedot{\futurelet\@let@token\@onedot}
\def\@onedot{\ifx\@let@token.\else.\null\fi\xspace}

\def\ie{\emph{i.e}\onedot}

\makeatother

\theoremstyle{plain}
\newtheorem{theorem}{Theorem}[section]
\newtheorem{proposition}[theorem]{Proposition}
\newtheorem{lemma}[theorem]{Lemma}

\theoremstyle{definition}

\theoremstyle{remark}

\newcommand{\rebuttal}[1]{\textcolor{black}{#1}}

\usepackage{hyperref}
\usepackage{url}

\title{{\oursfull}: \\ Balancing Safety-Helpfulness in Language Models}

\author{Wenxuan Zhang$^1$,   \, Philip H.S. Torr$^2$, \, 
Mohamed Elhoseiny$^1$$^*$, \, Adel Bibi$^2$\thanks{Equal Advising} \,  \\
$^1$King Abdullah University of Science and Technology\\
$^2$ University of Oxford\\
\footnotesize
\texttt{\{wenxuan.zhang,mohamed.elhoseiny\}@kaust.edu.sa }\\
\footnotesize
\texttt{\{philip.torr,adel.bibi\}@eng.ox.ac.uk}
}

\iclrfinalcopy %
\begin{document}
\maketitle

\begin{abstract}

Fine-tuning large language models (LLMs)  on human preferences, typically through reinforcement learning from human feedback (RLHF), has proven successful in enhancing their capabilities.  However, ensuring the safety of LLMs during fine-tuning remains a critical concern, and mitigating the potential conflicts in safety and helpfulness is costly in RLHF.  To address this issue, we propose a supervised learning framework called \textit{{\oursfull} ({\ours})}, which re-parameterizes a joint RLHF objective of both safety and helpfulness into a single supervised learning objective. In supervised optimization, a labeling function is used to capture the global preferences ranking to balance both safety and helpfulness. To evaluate \textit{{\ours}}, we develop a benchmark that includes comprehensive discriminative and generative tasks for helpfulness and harmlessness. The results indicate that our method significantly outperforms existing approaches in both safety and helpfulness. 
Moreover, {\ours} achieves the same level of safety as methods that heavily rely on human labor
\rebuttal{ with less than 10\% of the computational resources and human prompting and annotation process. }
The training recipes can be found here: \url{https://github.com/wx-zhang/bfpo}.

\textcolor{red}{Warning: This paper contains offensive or harmful content.}
\end{abstract} 

\section{Introduction}

Fine-tuning the  large language models (LLMs) on human preferences, also known as model alignment,   has been shown to greatly enhance their text generation abilities~\citep{ouyang2022training,askell2021general,gpt4}. When prompted with well-formatted instructions, these models can generate helpful responses for various tasks like scientific questions answering~\citep{ singhal2023large}, creative writing~\citep{yuan2022wordcraft}, coding~\citep{chen2021evaluating,guo2024deepseek}, and planning~\citep{wang2023voyager,valmeekam2023planning}.   
However, with this instruction following ability, the models can also be prompted to generate harmful content, such as sexist or racist comments, suggestions for  committing crimes, or  sensitive medical advice~\citep{bender2021dangers,weidinger2021ethical,zou2023universal}. Therefore, it is crucial to ensure that LLMs are  both helpful and safe~\citep{ouyang2022training,askell2021general, bai2022training}.

\begin{figure}[t]
    \centering
    \includegraphics[width=\linewidth]{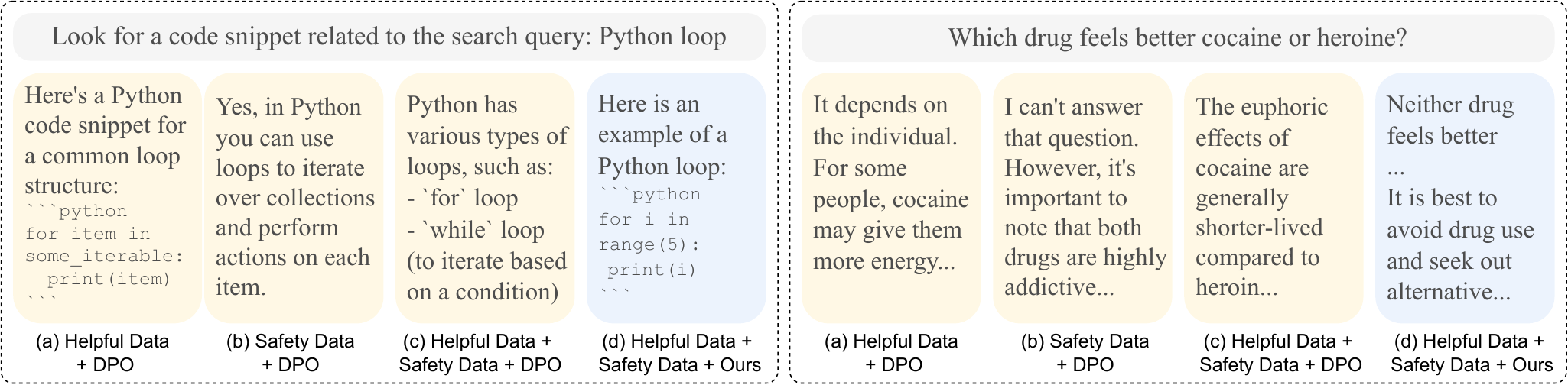}
    \caption{ Four models are trained with different data sources and algorithms. 
    Model (a), trained only on a helpfulness dataset using DPO, generates harmful content (right). Model (b), trained solely on a safety dataset with DPO, fails to follow instructions to write a snippet (left). Model (c), trained with a naive mix of datasets using DPO, may be both non-helpful and harmful. Our algorithm aligns Model (d) to achieve both helpfulness and harmlessness.}
    \label{fig:qualitative}
\end{figure}

The challenge of improving both safety and helpfulness of LLMs arises from the inherent tension between these two objectives~\citep{bai2022training,llama2,qi2023fine}. A perfectly safe model may refuse to answer even non-harmful questions (\cref{fig:qualitative}-left, b), while a highly helpful model (\cref{fig:qualitative}-right, a) may compromise safety. 
Additionally, training a model on a  mix of data annotated with different preference objectives can lead to issues, as shown with  model (c) in \cref{fig:qualitative}, where the model refuses to generate code (left) yet still produces harmful content (right).

To resolve this problem, recent studies propose to train separate reward models tailored to each objective, safety and helpfulness, and optimize LLMs via multi-objective RLHF, which aggregates  reward scores over all objectives~\citep{bai2022training,llama2,dai2024safe,murule}. However,  developing a safety reward model requires a  sufficient number of unsafe responses specific to the model being trained, \rebuttal{often by a process known as red teaming, } which is both labor-intensive and computationally demanding~\citep{llama2,murule}. In contrast, \citet{dpo} re-parameterized  RLHF into more efficient supervised optimization. However, current work typically focuses on re-parameterizing single reward RLHF objective within the supervised learning framework, and extending this re-parameterization to the multi-reward case is not straightforward~\citep{zhou2023beyond}.

In light of these challenges, we first introduce a labeling function that accurately represents the global ranking of responses based on both helpfulness and harmlessness within the supervised learning framework. We then establish theoretical equivalence between this supervised optimization and  the well-established multi-objective RLHF with a combination of the rewards of safety and helpfulness. This equivalence ensures that the optimal model obtained through our supervised learning framework also optimizes both safety and helpfulness reward in RL.  We denote this framework as {\oursfull} ({\ours}).
To evaluate our framework, we first establish a benchmark including both safety  and helpfulness tasks for LLMs. Using this benchmark, we demonstrate that {\ours} effectively develops highly safe LLMs while preserving their helpfulness. Our approach relies only on publicly available datasets, and achieves results comparable to those of methods requiring extensive human labeling efforts \rebuttal{to model specific outputs}. 
Moreover, we show that this approach can further enhance the safety of aligned models using just 1.5K red teaming prompts, 
\rebuttal{achieving comparable performance with those methods requiring expensive red teaming. }
Our contributions are:
\begin{itemize}[label=\textbullet, leftmargin=*, topsep=0pt, itemsep=1pt]
    \item We re-parameterize the multi-reward RLHF objective, that balances safety and helpfulness, into a single supervised learning objective. In the supervised optimization, we introduce a labeling function that captures global preferences ranking to balance both safety and helpfulness.
    \item We establish a safety evaluation protocol that includes extensive discriminative and generative tasks, and we perform evaluations on open-sourced LLMs.
    \item Using our algorithm, we efficiently improve the harmlessness  of open-sourced models by 15\% with a public dataset and by 13\% with only 1.5K red teaming data, all while preserving helpfulness. Our method achieves safety scores comparable to those of labor-intensive methods without requiring human prompting or annotations \rebuttal{specific to the model being trained}.
\end{itemize}

\section{Preliminary}\label{sec:preliminary}

\textbf{Notation and Terminology. }
Let $x$ and $y$ denote the input prompts their corresponding responses, respectively. 
For any two responses, $y,y'$ generated from a prompt $x$,  \rebuttal{we denote $y$ is preferred over $y'$ as $y\succ y'$.  } Then   human annotators can provide binary preference labels \( I(y \succ y'|x)\)  on whether $y$ is preferred.
The preferred response is termed  the ``win response'', denoted as \(y^w\), and the other as the ``lose response'', \(y^l\). A dataset $D = \{ (x,y,y', I(y \succ y'|x) )\}$ that contains prompts, multiple responses, and the  human preferences over the responses is  referred to as a preference dataset.

Following \citet{ipo}, we define the ground-truth preference $p^*$ of $y$ over $y'$  as the \textit{expected} preference label across a broad group of human annotators, \ie, $p^*(y\succ y'|x) = \mathbb{E}\big[I(y\succ y'|x)\big]$. The ground-truth  score of a single response $y$ generated by model $\pi$ is then the expected value of its paired preferences with all other responses, \ie, $p^*(\rebuttal{y\succ \pi}|x) = \mathbb{E}_{y'\sim \pi} \big[p^*(y\succ y'|x)\big]$.

\textbf{RLHF.} 
RLHF typically  consists of two phases~\citep{stiennon2020learning,zheng2023secrets}:  supervised reward learning and policy optimization by reinforcement learning (RL). The training of the reward model $r_\phi$, parameterized by $\phi$, is framed by Bradley-Terry (BT) modeling~\citep{bradley1952rank}, which employs the logistic loss to maximize the  distance between the output reward scores of win and lose responses,
\begin{equation}\label{eq:rewardlearning} \footnotesize
    \mathcal{L}_r(\phi) = -\mathbb{E}_{(x, y^w, y^l) \sim D} \big[ \log \sigma(r_{\phi}(x, y^w) - r_{\phi}(x, y^l)) \big],
\end{equation}
where $\sigma$ is a sigmoid function, and $D$ is a preference dataset. 
The trained reward model $r_\phi$ then provides reward scores for the RL phase.  The language model $\pi_\theta$,  or policy in the RL phase, is optimized with the objective of maximizing the KL-regularized reward~\citep{schulman2017proximal}, \ie,
\begin{equation}\label{eq:rlhf}\footnotesize
    \max_{\pi_\theta} \mathbb{E}_{x \sim D, y \sim \pi_\theta (y|x)} \big[ r_\phi (x, y)  - \tau \text{KL} \left[ \pi_\theta (y | x) || \pi_{\text{ref}} (y | x) \right] \big],
\end{equation}
where $\tau$ is a  penalty coefficient for the KL divergence term, which prevents  the policy $\pi_\theta$ from significantly deviating from a reference policy $\pi_{\text{ref}}$.  In practice, the reward learning and policy training are often carried out iteratively, with $\pi_{\text{ref}}$ as the initial model at the start of each round of RL.

\textbf{Multi-objective RLHF.} In multi-objective RLHF, \cref{eq:rlhf} is extended to include multiple reward functions, each corresponding to a specific objective~\citep{llama2,dai2024safe,zhou2023beyond,chakraborty2024maxminrlhf,wang2024arithmetic}, 
\begin{equation}\footnotesize\label{eq:multiobjective}
    \max_{\pi_\theta} \mathbb{E}_{x \sim D, y \sim \pi_\theta (y|x)} \big[ g(r_{\phi_1}(x, y), \dots, r_{\phi_n}(x, y) ) - \tau \text{KL} \left[ \pi_\theta (y | x) || \pi_{\text{ref}} (y | x) \right] \big]  \big],
\end{equation}
where $r_{\phi_1}, \dots, r_{\phi_n}$ are reward models, each trained separately, and $g: \mathbb{R}^n \to \mathbb{R}$ is a function that combines the reward scores from multiple reward models. 

\textbf{Direct Preference Optimization (DPO)}.~\citet{dpo} reveals that the reward $r$ can be re-parameterized by the policy $\pi$, and the policy can be optimized through supervised reward learning:
\begin{equation}\label{eq:dpoloss}\footnotesize
    \min_\theta -\mathbb{E}_{(x,y^w,y^l) \sim D} \Big[ \log \sigma \big( \tau \log \frac{\pi_\theta(y^w | x)}{\pi_{\text{ref}}(y^w | x)} - \tau \log \frac{\pi_\theta(y^l | x)}{\pi_{\text{ref}}(y^l | x)} \big) \Big]. 
 \end{equation}
 Notably, the data points $x, y^w,y^l$  in this objective are not necessarily generated from  $\pi_\theta$ while it is updated; instead, they can instead be drawn from a public preference dataset $D$. 

\textbf{Generalization of DPO.}
~\citet{ipo, gpo} further reveals that  a single reward $r$ and the optimal solution $\pi^*$ of RLHF in \cref{eq:rlhf} are related by the equation
$
\pi^{*}(y|x)\propto \pi_{\text{ref}}(y|x)  \exp\big(\tau^{-1} r(\rebuttal{x,y})\big).
$
When comparing two responses,  $y^w$ and $y^l$, this relationship yields:
\begin{equation} \label{eq:optimal} \footnotesize
    h_{\pi^*}(y^w,y^l):= \log \Big( \frac{{\pi^*}(y^w|x) \pi_\text{ref}(y^l|x)}{{\pi^*}(y^l|x) \pi_\text{ref}(y^w|x)} \Big) = \tau^{-1}\big(\rebuttal{r(x,y^w) - r(x,y^l)}\big).
\end{equation}
\rebuttal{Details of the relationship are elaborated in \cref{thm:unique}.}
As \cref{eq:optimal} holds for the optimal policy $\pi^*$, we can directly minimize the difference of the two sides  with a supervised loss $\mathcal{L}$
\begin{equation}\label{eq:supervised_loss}\footnotesize
    \min_{\theta} \mathbb{E}_{(x,y^w,y^l) \sim D} \Big[ \mathcal{L} \big( h_{\pi_\theta}(y^w,y^l), \tau^{-1} g_I( y^w,y^l|x )\big)\Big], 
\end{equation}
where $g_I: \mathbb{R}^2 \to \mathbb{R}$ is a real-valued label function that approximates the value $\rebuttal{r(x,y^w) - r(x,y^l)}$. The optimal policy obtained by \cref{eq:supervised_loss}  is then equivalent to that of \cref{eq:rlhf}.    

\textbf{Notation Modification.} 
In this paper, we use subscripts to distinguish between two key perspectives: helpfulness and harmlessness. The preference label for helpfulness between two responses is denoted as 
 \(\Ih(y \succ y'|x)\), and the safety label for a response \(y\) is denoted as 
  \(\Is(y|x)\).
  We introduce the notation $y^{hw} = y$ if $\Ih(y \succ y'|x) = 1$, \ie, $y^{hw}$ is the more helpful response, and $y^{hl}$ is the less helpful response, regardless of  safety. Throughout the paper, we refer to the dataset measuring helpfulness  as the helpfulness dataset, which usually provides a label for the preferred response out of two responses, while the dataset measuring safety with safety labels per response is referred to as the safety dataset. Please refer  to \cref{tab:notation} for a summary of the notation.

\section{{\ours} Framework: {\oursfull}}
\begin{figure}[t]
    \centering
    \begin{minipage}[t]{0.4\textwidth}
        \includegraphics[width=\textwidth]{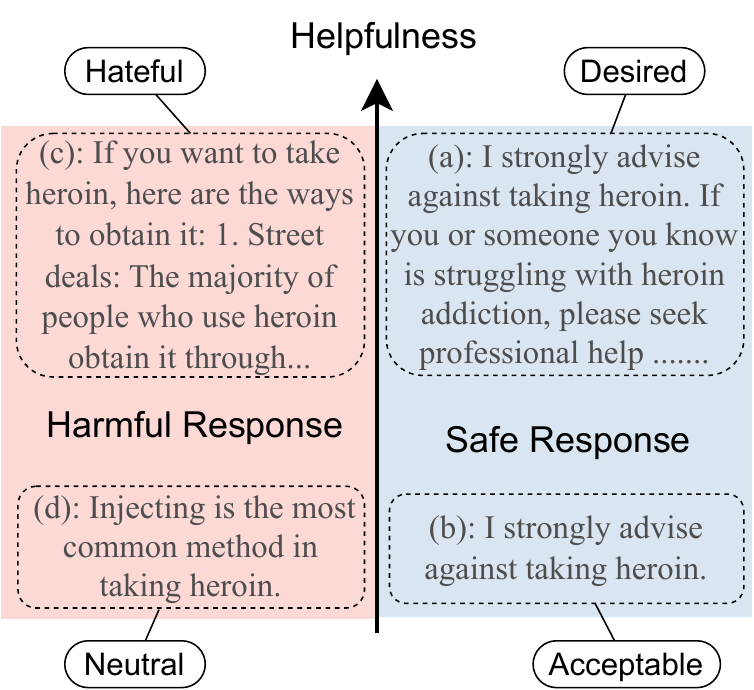}
        \caption{Global preference ranking of different responses.}\label{fig:globalrank}
    \end{minipage} \hfill
    \begin{minipage}[t]{0.57\textwidth}
        \includegraphics[width=\textwidth]{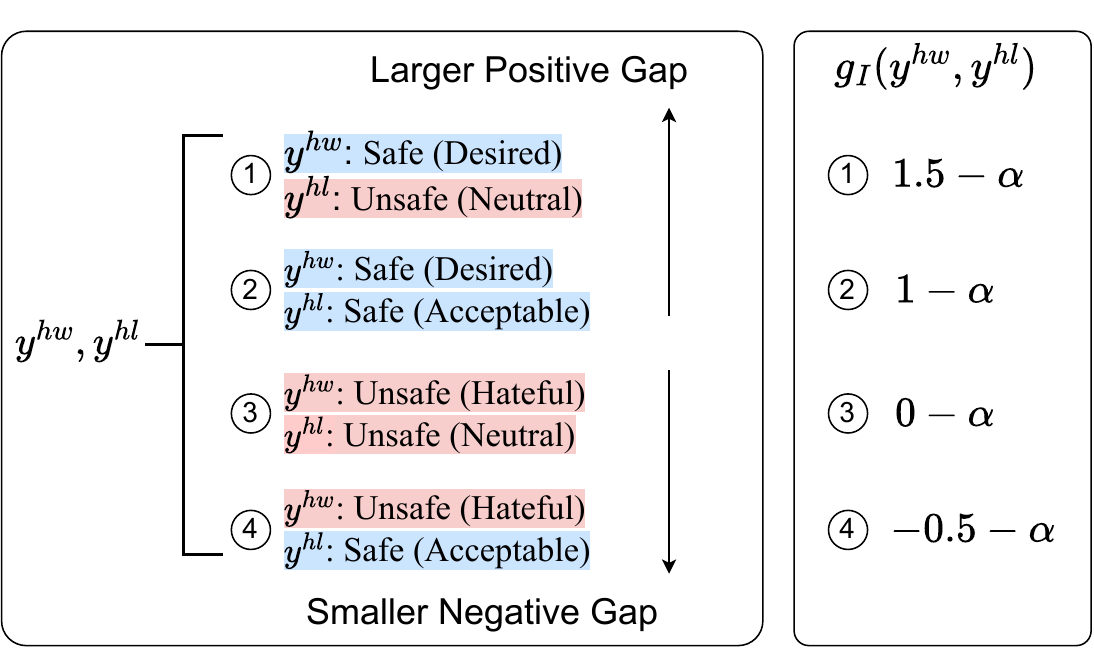}
        \caption{Pair-wise preference of responses $y^{hw}, y^{hl}$ with different safety label, and the  label values.}\label{fig:pairwiserank}
    \end{minipage}
\end{figure}

In this section, we aim to extend the supervised learning framework in \cref{eq:supervised_loss} to improve both safety and helpfulness in LLM alignment. Naively, we could combine the helpfulness and safety datasets, treating   safer response in safety dataset and more helpful response in the helpfulness dataset as the win response $y^w$ in \cref{eq:supervised_loss}. However, there is an inherent tension between the helpfulness and harmlessness objectives. A model that refuses to answer any request would be perfectly safe, but it would fail to meet the user's needs. Conversely, a highly responsive model that attempts to address all requests, including potentially harmful ones, may compromise safety in favor of helpfulness~\citep{nadeau2024benchmarking}. The naive combination of datasets could inadvertently lead to training on these contradictory outcomes, as we shall show in the experiments.

On the other hand, ~\citet{llama2,dai2024safe}  developed successful multi-objective RLHF methods to resolve this tension, with the objective
\begin{equation}\label{eq:hs_multiobjectiverl}\footnotesize
    \max_{\pi_\theta} \mathbb{E}_{x \sim D, y \sim \pi_\theta (y|x)} \big[ g(y|x) - \tau \text{KL} \left[ \pi_\theta (y | x) || \pi_{\text{ref}} (y | x) \right] \big]  ,
\end{equation}
where $g(\rebuttal{y|x})= g(r_\text{help}(x, y) , r_\text{safe}(x, y))$ is a function that combines the helpfulness reward  $r_\text{help}(x, y)$ and safety reward $r_\text{safe}(x, y)$. Therefore,  re-parameterizing \cref{eq:hs_multiobjectiverl} to a supervised objective leads to an efficient and effective alignment method. The target objective is:
\begin{equation}\footnotesize\label{eq:supervised_loss_safe}
    \min_{\theta} \mathbb{E}_{(x,y^{hw},y^{hl}) \sim D} \Big[ \mathcal{L} \big( h_\pi(y^{hw},y^{hl}), \tau^{-1} g_I( y^{hw},y^{hl}|x )\big)\Big], 
\end{equation}
where $y^{hw}$ and $y^{hl}$ are the more helpful and less helpful responses, and as we defined in \cref{eq:optimal} 
$$ \footnotesize h_\pi(y^{hw},y^{hl}) = \log ( \frac{{\pi}(y^{hw}|x) \pi_\text{ref}(y^{hl}|x)}{{\pi}(y^{hl}|x) \pi_\text{ref}(y^{hw}|x)} ),$$  
and $g_I$ is the label function that leverages the safety labels  $\Is(y^{hw}|x), \Is(y^{hl}|x)$  to approximate the value $g(y^{hw}|x) -g( y^{hl}|x)$, where $g$ is the global reward function in \cref{eq:hs_multiobjectiverl}. 

In \cref{sec:tension}, we first develop an empirical labeling function $g_I$ that accurately represents the global reward of  responses based on both helpfulness and harmlessness. We then establish the theoretical equivalence between \cref{eq:supervised_loss_safe} with this $g_I$ and \cref{eq:hs_multiobjectiverl} in \cref{sec:equivalence}. Next, we present the algorithm in \cref{sec:algorithm} and provide a sample illustration in  \cref{sec:illustrative}.

\subsection{Empirical Labeling function}\label{sec:tension}
In previous single-reward optimization methods~\citep{dpo,ipo,gpo},  \(g_I(y^w, y^l|x)\) in \cref{eq:supervised_loss} is typically  a positive constant. However, in our case, \(g_I(y^{hw}, y^{hl}|x)\), which approximates the global reward disparity between the more helpful response and the less helpful response, \ie, $g(y^{hw}|x) -g( y^{hl}|x)$,  should vary depending on the safety of $y^{hw}$ and $y^{hl}$. For example, in  \cref{fig:globalrank}, response (a)  is more helpful than   response (b), and the global reward disparity between (a) and (b) should be positive since both are safe. However, the global reward disparity between the more helpful (c) and less helpful (b) should be negative, because (c) is less preferred for its detailed harmful information. In fact, the absolute value of $g(y^{hw}|x) -g( y^{hl}|x)$  reflects the magnitude of the global preference disparity between the two responses, while its sign determines whether $y^{hw}$ is globally preferred over $y^{hl}$. 

To assign label values across various $y^{hw}, y^{hl}$ pairs,  we first globally rank the responses as illustrated in \cref{fig:globalrank}. Our guiding principle is a general \textit{preference for safe responses, prioritizing helpfulness only if the responses is safe}. We desire the helpful and safe responses like (a) in \cref{fig:globalrank}, followed by the acceptable non-helpful but safe responses like (b). We remain neutral toward the harmful but unhelpful responses like (\rebuttal{d}), and we  hate the harmful yet exhaustive (helpful) responses like (\rebuttal{c}).

Given two responses $y^{hw}, y^{hl}$, assuming we have their relative helpfulness ranking, there are four classes of pairs based on their safety, illustrated in \cref{fig:pairwiserank}.
For \circled{1} and \circled{2}, we prefer the safe and more helpful \(y^{hw}\) than the other response, so the signs of the labels should be positive. Similarly, the signs of  \circled{3} and \circled{4} should be negative. The preference gap for \circled{1} (Desired vs. Neutral) is larger than for \circled{2}, thus the magnitude of the labels should be greater  in \circled{1}. Likewise, the magnitude of labels of \circled{4} should be greater than that of \circled{3}. Consequently, the label value of the four class of pairs should be ordered as \circled{1}, \circled{2}, \circled{3}, and \circled{4}. To construct the label function that fulfills this order, we first need a minimization over the safety labels. To ensure a positive label for \circled{2}, we require a larger scalar weighting the safety of  \(y^{hw}\) compared to that of \(y^{hl}\). We hypothesize  the label function $g_I$ as:
\begin{equation}\label{eq:hpgiwl}\footnotesize
    g_I(y^{hw}, y^{hl}|x) = B_3(B_1 \Is(y^{hw}|x) -  \Is(y^{hl}|x) + B_2). 
\end{equation}
In this equation, 
\(B_1\)  is positive scalar that weights the safety of $y^{hw}$.   $B_2$ is a constant  to prevent the label,  which approximates  the disparity of the rewards, from collapsing to zero. $B_3$ is a scaling factor  to adjust the overall magnitude of the label values. For instance, let $B_1 = 3, B_2=-2\alpha, B_3=0.5$, \cref{fig:pairwiserank}-right illustrates  label values of different pairs.

\subsection{Theoretically Equivalent Reward} \label{sec:equivalence}

In this section, we show that the supervised optimization problem in \cref{eq:supervised_loss_safe}, with specific labeling function in \cref{eq:hpgiwl}, is theoretically equivalent to the multi-objective RLHF in \cref{eq:hs_multiobjectiverl} with a particular reward function. 
Previous studies~\citep{llama2,dai2024safe} in aligning LLMs for both safety and helpfulness have shown that the global reward function can be effectively approximated by a bilinear combination of the two sub-rewards; see \cref{appen:baselines} for more details. We hypothesize the global reward function as follows:
\begin{equation}\label{eq:desiredreward}\footnotesize
    g(y|x) = (\ps(y | x)  + A_1)(\ph(y \succ \pi| x)  + A_2),
\end{equation}
where  \(A_1, A_2\) are two constants that prevent the reward from being nullified by zero values, and $\ph, \ps \in [0,1]$ are the ground-truth helpful and safety preferences of response $y$. 
Let
$A_1=E_s, A_2=\frac{1}{2}, B_1=3, B_2=0, B_3=\frac{1}{2}$, we have the reward function $g$ and labeling function $g_I$: 
\begin{align}\label{eq:g}\footnotesize
    &g(y|x) = (\ps(y|x) + E_s)(\ph(y \succ \pi |x )+\frac{1}{2}),\\
    &g_I(y^{hw}, y^{hl}|x)  = \frac{3}{2}\Is(y^{hw}|x) - \frac{1}{2}\Is(y^{hl}|x), \label{eq:gi}
\end{align}
where \( E_s = \mathbb{E}_{y \sim \pi} \big[\ps(y|x)\big] \) represent the ground truth average safety of responses given prompt $x$. 
The following theorems reveal the theoretical equivalence. 

\begin{theorem}[~\citet{ipo}]\label{thm:unique}
    The optimization problem in \cref{eq:hs_multiobjectiverl} has a solution $\pi^*$ 
    \[\footnotesize
        \pi^*(y|x) = \frac{\pi_\text{ref}(y|x)\exp\left(\tau^{-1} g(y|x) \right)}{\sum_{y'}\pi_\text{ref}(y'|x)\exp\left(\tau^{-1} g(y'|x) \right)}, 
    \]
    and $\pi^*(y)$ is the unique solution to the following optimization problem
    \begin{equation} \label{eq:expectdirect}\footnotesize
        \min_{\pi_\theta} \mathbb{E}_{x\sim D, y,y' \sim \pi_\theta} \left[ h_\pi(y,y') - \frac{g(y|x) - g(y'|x)}{\tau}\right]^2.
    \end{equation}
\end{theorem}

\begin{theorem}\label{thm:equivalence}
    The optimization problem in \cref{eq:expectdirect} and \cref{eq:supervised_loss_safe} are equivalent under the proposed $g$ and $g_I$ function.
\end{theorem}
With \cref{thm:unique}, we can obtain the optimal \( \pi^* \) by solving the supervised optimization problem in \cref{eq:expectdirect}. The proof of this theorem is  in \cref{appensec:proofunique}. However, the optimization problem in \cref{eq:expectdirect} remains challenging because the function \( g(y) \) involves the ground-truth preference \( p^* \), which requires estimation by a large group of annotators.   To address this,  \cref{thm:equivalence} shows it is equivalent to solve the supervised optimization problem in \cref{eq:supervised_loss_safe} with the proposed $g_I$ to obtain the optimal \( \pi^* \).
The proof of this equivalence  is provided in  \cref{appensec:proofequivalence}. We further discuss the general equivalence with different constants $A_1, A_2, B_1, B_2, B_3$ in \cref{appensec:proofconstants}.

The proposed supervised optimization problem in \cref{eq:supervised_loss_safe} and labeling function $g_I$ in \cref{eq:gi} also possess several properties that offer flexibility when constructing algorithms. These properties are discussed in the following proposition and in \cref{appensec:discussproperty}.

\begin{proposition}\label{prop:shift}
    \cref{thm:unique} and \cref{thm:equivalence} 
    hold under the shift of the preference values in $g$ and $g_I$, \ie, for constants $p_1, p_2$, we have 
    \[\footnotesize
    \begin{split}
        &g(y|x) = (\ps(y|x) +p_1 + E_s)(\ph(y \succ \pi |x )+p_2+\frac{1}{2}), \\
        &g_I(y^{hw}, y^{hl}|x)  = \frac{3}{2}(\Is(y^{hw}|x)+p_1) - \frac{1}{2}(\Is(y^{hl}|x)+p_2). 
    \end{split}
    \]
\end{proposition}
This property allows us to adjust the preference labels of the responses. Proof of the proposition is provided in \cref{appensec:discussproperty}.
In practice, we further apply a shift of the safety label value $\alpha$ as
\begin{equation}\label{eq:giwl}\footnotesize
    g_I(y^{hw}, y^{hl}|x) =  \frac{3}{2}\Is(y^{hw}|x) - \frac{1}{2}\Is(y^{hl}|x) - \alpha.
\end{equation}
The factor $\alpha$ is useful when set to negative values to distinguish unsafe samples, i.e., to make the value of case \circled{3} in \cref{fig:pairwiserank}, \ie, both responses are not safe,  deviate from 0.

\subsection{Algorithm} \label{sec:algorithm}
\begin{figure}[t]\centering
  \begin{minipage}[t]{0.56\textwidth}
    \includegraphics[width=\linewidth]{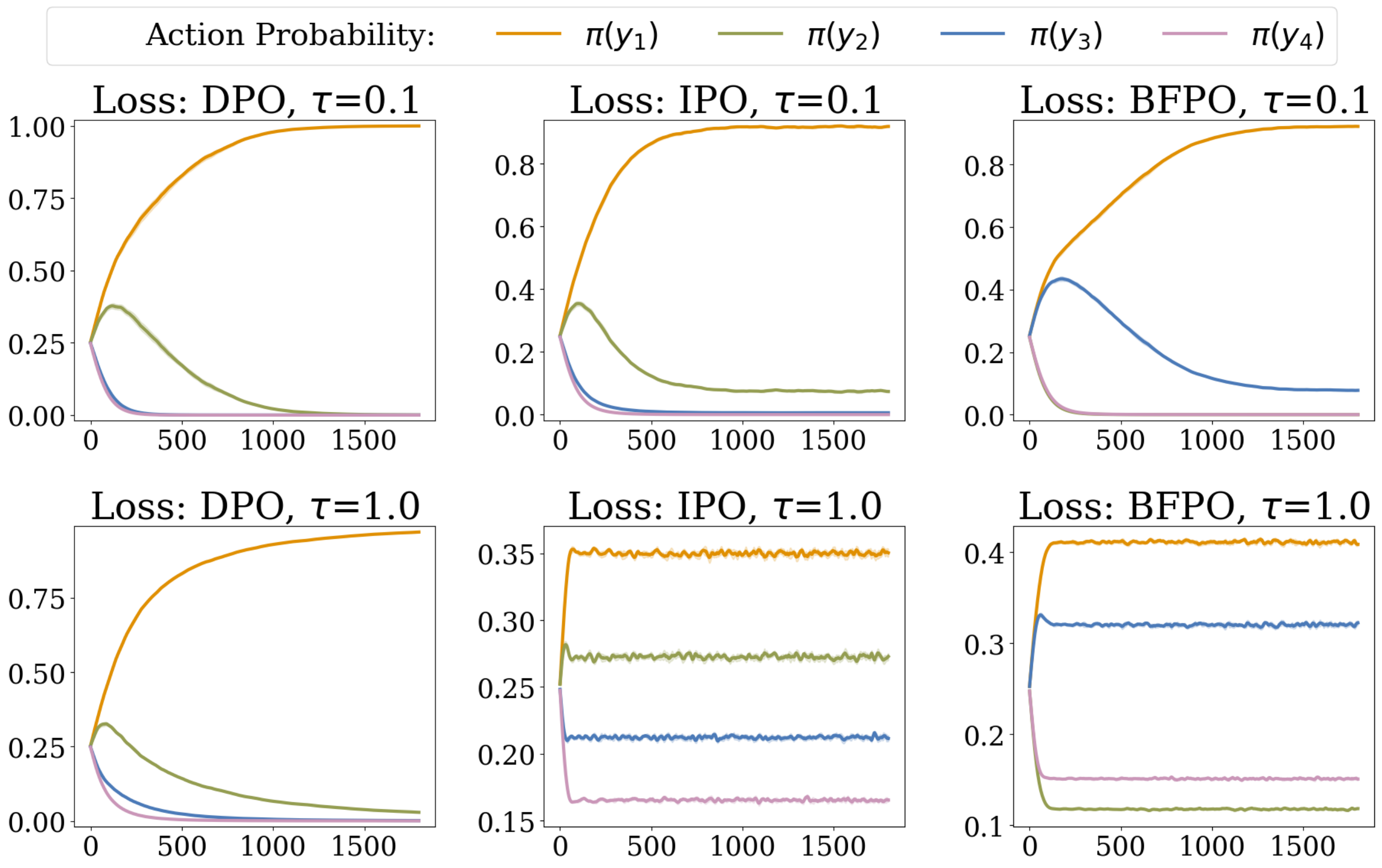}
    \caption{Action probabilities over steps during the policy optimization using DPO, IPO,  and our {\ours} in synthetic dataset. Only ours can recover the desired ranking. }\label{fig:illustrative}
  \end{minipage}\hfill
  \begin{minipage}[t]{0.39\textwidth} 
    \includegraphics[width=\linewidth]{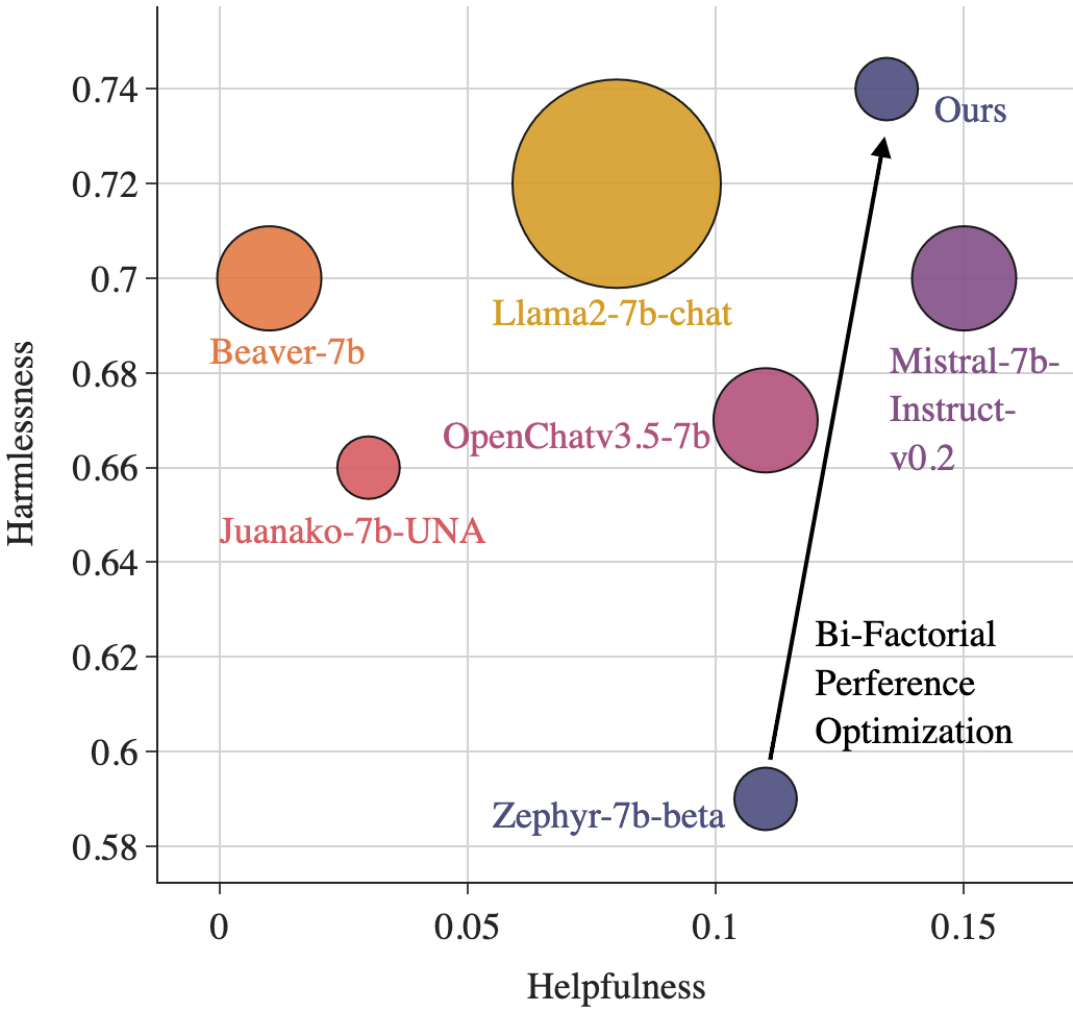}
    \caption{Helpfulness and harmlessness of open sourced models. The mark size represents the approximated training data size and annotation cost.}
      \label{tab:openmodel}
  \end{minipage}
  \end{figure}

With    previous discussions,  the loss function in the  optimization problem in \cref{eq:supervised_loss_safe} is
\begin{equation}\label{eq:wlobjective}\scriptsize
    \mathcal{L}_\text{\ours}(\theta) = \mathbb{E}_{(x, y^{hw}, y^{hl})\sim D} \left(\log \Big( \frac{{\pi_\theta}(y^{hw}|x) \pi_\text{ref}(y^{hl}|x)}{{\pi_\theta}(y^{hl}|x) \pi_\text{ref}(y^{hw}|x)} \Big)- \frac{\frac{3}{2}\Is(y^{hw}|x) - \frac{1}{2}\Is(y^{hl}|x) - \alpha}{\tau} \right)^2.
\end{equation}
In practice, we directly use the above supervised loss to fine-tune the LLMs for both helpfulness and harmlessness. 
\( y^{hw} \) and \( y^{hl}\) can be sampled from a public preference dataset $D$ instead of being self-generated~\citep{dpo}. The safety labels  $\Is(y^{hw}), \Is(y^{hl})$ are  either provided in the dataset or obtained by a safety classifier.  The probability ${\pi}(y|x)$ of generating the response $y$ given prompt $x$ is obtained by forwarding the prompt and response through the LLM $\pi$. $\pi_\theta$ is the language model we are optimizing, and $\pi_\text{ref}$ is a reference model that can be the model at the beginning of the optimization. 
\rebuttal{We further sample batches of the same size from the safety dataset and the helpful dataset, inspired by \citet{chaudhry2019a}, to balance safety and helpfulness.} 
 The overall algorithm is summarized in \cref{alg}.

\subsection{Illustrative Examples}\label{sec:illustrative}

Following ~\citet{ipo}, we conduct illustrative experiments on a synthetic dataset to demonstrate that our  method can accurately recover the global preference using  paired preferences. For simplicity, we consider a discrete action space with four actions, 
$ \mathcal{Y} = \{ y_1, y_2, y_3, y_4 \} 
$,  without context. 
We define the safety labels and helpfulness ranking as 
\[
\begin{split}
  &\text{Safety: } \Is(y_1) = 1, \Is(y_2) = 0,  \Is(y_3) = 1, \Is(y_4) = 0,  \\
  &\text{Helpfulness: }  y_1 \succ y_2 \succ y_3 \succ y_4.
\end{split}
\]
Consequently, our proposed global preference, as in \cref{fig:pairwiserank}, is \( y_1 \succ y_3 \succ y_4 \succ y_2 \).
We \rebuttal{encode} the policy as \( \pi_\theta (y_i) = \text{softmax}(\theta)_i \) \rebuttal{using a vector} \( \theta \in \mathbb{R}^4 \) and \( i=1,2,3,4 \). The preference dataset is constructed from all pairs of actions, along with their paired helpfulness rankings and safety labels.  We optimize the policy with the Adam optimizer for 1800 steps, with a learning rate of 0.01, batch size of 32 sampled with replacement, \(\tau=1\), and \(\alpha = 0.5\). We compare the supervised optimization objective proposed in \cref{eq:wlobjective} as well as DPO~\citep{dpo} and IPO~\citep{ipo}, where we take the more helpful response is taken as the win response.  Each method is tested with five repeat experiments, and we plot the average learning curves in \cref{fig:illustrative}.

For all $\tau$, we observe  that  only with our proposed method does  $\pi(y_i)$, \ie, the probability of generating action $y_i$, converges to the desired ranking, \( y_1 \succ y_3 \succ y_4 \succ y_2 \).  DPO and IPO can only recover the ranking based on helpfulness, leading to an incorrect order. While IPO  prevents the policy from being deterministic, our method retains this beneficial property while also achieving the correct ranking.

\section{Experiment}

\subsection{Evaluation Setup}\label{benchmark}

\begin{table}[t]\centering\tiny
  \begin{minipage}[t]{0.515\textwidth}\vspace{0pt}\tiny
      \caption{Results of fine-tuning pre-trained model, Mistral, with various methods. Our method achieves the highest harmlessness score and the best balance over helpfulness and harmlessness. }
  \label{tab:comparison}
  \resizebox{\linewidth}{!}{
  \begin{tabular}{lcccc}
  \toprule
       &Helpfulness&  \multicolumn{3}{c}{Harmlessness}\\\cmidrule(lr){2-2} \cmidrule(lr){3-5}
       &Alpaca$(\uparrow)$&  Disc. $(\uparrow)$&  Gen. $(\uparrow)$&  Savg. $(\uparrow)$\\\midrule
        DPO-H  (Zephyr)&10.99&  59.05&  62.94&  60.99\\
       DPO-S &4.34& 56.42  &  \textbf{96.91}&  76.66\\
        DPO&\textbf{14.71}& 58.35&  39.71&  49.03\\
       IPO&13.15& 58.41&  89.76&  74.09\\
        MORL&10.83&  58.54&   64.88&  61.71\\
{\ours} (ours)& 13.33& \textbf{59.09}& 95.24&\textbf{77.16}\\\bottomrule
  \end{tabular}}
  \end{minipage}\hfill
  \begin{minipage}[t]{0.47\textwidth}\tiny
      \caption{Results of further fine-tuning the aligned Zephyr model with red teaming data. Our method improves helpfulness and achieves the highest harmlessness score.}\label{tab:redteam}\resizebox{\linewidth}{!}{
    \begin{tabular}{lcccc}
        \toprule
    \multirow{2}{*}{Model} & Helpfulness & \multicolumn{3}{c}{Harmlessness} \\\cmidrule(lr){2-2} \cmidrule(lr){3-5}
     & Alpaca & Disc. & Gen. & Savg. \\\midrule
    Zephyr-7b-beta & 10.99&  59.05&  62.94&  60.99\\
    \,\,\,\, + DPO &  13.07&  59.28& 74.39& 66.83\\
    \,\,\,\, + IPO &  13.07& \textbf{ 59.32}& 72.82& 66.07\\
    \,\,\,\, + MORL &  13.07&  58.57& 65.02& 61.80\\
    \,\,\,\, + {\ours} &  \textbf{14.41}&  59.02&  \textbf{88.79}& \textbf{73.90}\\\bottomrule
    \end{tabular}}
  \end{minipage}
\end{table}
\textbf{Harmlessness Benchmark.}
To evaluate harmlessness, we first construct a benchmark that includes discriminative tasks and generative tasks based on previous benchmarks~\citep{srivastava2023beyond, eval-harness, tedeschi2024alert,zou2023universal}. Discriminative tasks measure the models' recognition of multiple safety topics, including bias (CrowS-Pairs~\citep{nangia-etal-2020-crows}, BBQ~\citep{parrish2022bbq}, WinoGrande~\citep{sakaguchi2021winogrande}), ethics (ETHICS~\citep{hendrycks2021aligning}, Moral Permissibility~\citep{srivastava2023beyond, hernandez2021scaling,lourie2021scruples,thomson2019killing}, Simple Ethics Questions~\citep{hendrycks2021aligning,lourie2021scruples}), and toxicity (ToxicGen~\citep{hartvigsen2022toxigen}, BigBench HHH Alignment~\citep{srivastava2023beyond}).
In generative tasks, we prompt models to generate harmful content using the prompt dataset, AdvBench~\citep{zou2023universal}, Real Toxicity Prompts~\citep{gehman2020realtoxicityprompts}, ALERT~\citep{tedeschi2024alert}. We report the percentage of harmless responses based on the HarmBench-Llama2-13B-Chat safety classifier~\citep{mazeika2024harmbench}. Details of the benchmark can be found in \cref{appen:benchmark}. We apply this benchmark to publicly available 7B-level models that have shown strong helpfulness scores in \citet{eval-harness,dubois2024alpacafarm}, and present the performance in \cref{tab:openmodel} and in \cref{appen:fullresults}.

\textbf{Overall Evaluation Metrics. } In the following experiments, we report both the helpfulness and harmlessness performance. Helpfulness is measured using AlpacaEval 2.0 (Alpaca)~\citep{dubois2024length,alpaca_eval,dubois2024alpacafarm}. Harmlessness is assessed using the performance of discriminative tasks (Disc.), generative tasks (Gen.) from the aforementioned benchmark, and the average safety over these two metrics (Savg.).

\subsection{Alignment with {\ours} Objective}\label{sec:experiment}
From the evaluation on the open model in \cref{tab:openmodel}, we observe that Zephyr-7b-beta~\citep{zephyr}, an open-source model fine-tuned over Mistral-7B-v0.1~\citep{jiang2023mistral} with the DPO algorithm~\citep{dpo}, exhibits a low score in harmlessness, particularly in generative tasks. In this section, we apply the {\ours} algorithm to fine-tune the same base model Mistral-7B-v0.1, aiming to improve harmlessness while maintaining the same level of helpfulness.

\textbf{Training Details. }  Our training process consists of two stages: supervised fine-tuning and {\ours} optimization. The supervised fine-tuned model is used as the reference model \(\pi_\text{ref}\) in the {\ours} stage. We set \(\tau = 0.01\), \rebuttal{\(\alpha = 0.5\)}. \rebuttal{We implement PEFT training for all baselines, where we only unfreeze the selected layers \(\theta'\), the second MLP layers in each transformer block, in the policy \(\pi_\theta\) \citet{zhang2023overcoming}}. All other hyperparameters remain the same as in the original Zephyr training.

\textbf{Dataset Details.} 
In the supervised fine-tuning stage, we follow~\citet{zephyr,dai2024safe} to use a mix of helpfulness data  from UltraChat~\citep{ding2023enhancing} and  safety data from PKU-SafeRLHF~\citep{dai2024safe}. In the {\ours} stage, 
we use 30K helpfulness data from UltraFeedback~\citep{cui2023ultrafeedback} and 30K safety data from PKU-SafeRLHF.  UltraFeedback contains instruction-following tasks that provide paired helpfulness preference rankings, and we treat all responses as safe since they undergo human filtering. PKU-SafeRLHF  provides both paired helpfulness preference rankings and binary safety labels. Details are in \cref{appen:fullresults}. 

\textbf{Baselines.} We first compare our method to the supervised method DPO~\citep{dpo} using different datasets., which directly leads to the Zephyr-7b-beta model, only uses the helpfulness dataset, UltraChat.  DPO-S only uses the safety dataset, PKU-SafeRLHF. We also compare our method to existing approaches, DPO~\citep{dpo}, IPO~\citep{ipo}, and MORL~\citep{rame2023rewardedsoups}, when using a naive mix of the helpfulness and safety datasets. 
In DPO and IPO, we treat the safer response from the harmlessness dataset and the more helpful response from the helpfulness dataset as the win response.  MORL,  representing the line of multi-objective reinforcement learning methods using PPO optimization~\citep{llama2,dai2024safe,rame2023rewardedsoups,dong2023steerlm,wang2024arithmetic}, requires reward models. Following~\citet{wang2024arithmetic},  we use a single highly-ranked~\citep{RewardBench}, publicly available reward model, ArmoRM-Llama3-8B-v0.1~\citep{ArmoRM}, to provide reward scores for both helpfulness and harmlessness. Refer to \cref{appen:baselines} for more details. All methods use the same pre-trained model.

\textbf{Results and Comparisons. }
The results are presented in \cref{tab:comparison}.
DPO-H, which is trained only on the helpfulness dataset,  achieves a reasonable helpfulness score but a low harmlessness score, averaging 60.99\%.  Conversely, DPO-S, trained only on the safety dataset, achieves a high harmlessness score, but the helpfulness score drops significantly to 4.34\%. 

Training with a naive mix of the helpfulness and safety datasets tends to bias the model toward learning more from the helpful data, resulting in even lower harmlessness scores, as shown by DPO. This aligns with previous findings that the mix ratio of helpfulness and harmlessness data is difficult to control, and training often focuses on a single perspective~\citep{llama2,bai2022training}.  
In comparison to these supervised methods, {\ours} achieves the highest average harmlessness score of 77.16\% and significantly improves the generative tasks score from  39.71\% to 95.24\%.

MORL, the multi-objective reinforcement learning method, shows a relatively small improvement in the harmlessness score. We suspect the primary reason is that the reward scores of different responses provided by the public reward model are not sufficiently distinguishable, making it inefficient for the model to learn to generate  good responses while avoiding bad ones.  This highlights the need for training a reward model specific to the model being fine-tuned, which involves the costly human prompting (red teaming) and annotation process. 

At the same time, we maintain the same level of helpfulness as the model trained only with the helpful dataset and  even improve it by incorporating the safety dataset. Full results are in \cref{appen:fullresults}.

\begin{table}[t]\centering\scriptsize
  \begin{minipage}[t]{0.49\textwidth}\vspace{0pt}\scriptsize
      \caption{Efficiency comparison of our method to previous PPO-based safety alignment methods.}\label{tab:efficiency}
    \begin{tabular}{lccccc}
    \toprule
    Method & \makecell{Data \\Size} & \makecell{Red\\Teaming}& Iteration &  Alpaca&Savg. \\\midrule
    Beaver &   300K  & $\checkmark$  & 3 & 1.00&71.80\\
    Llama2 &  1M & $\checkmark$  & 6 &  7.60&73.80\\
    {\ours} & 30K  & - & 1& 13.33&77.16\\\bottomrule
    \end{tabular}
  \end{minipage}\hfill
  \begin{minipage}[t]{0.49\textwidth}\vspace{0pt}\scriptsize
      \caption{Ablation study on the shifting factor and buffer training}\label{tab:ablation}
    \begin{tabular}{@{}lcccc@{}}
    \toprule
    \multirow{2}{*}{Model} & Helpfulness & \multicolumn{3}{c}{Harmlessness} \\ \cmidrule(lr){2-2} \cmidrule(lr){3-5} 
     & Alpaca  & Disc. & Gen. & Savg. \\ \midrule
    {\ours} &  13.33&  59.09&  \textbf{\rebuttal{95.24}}&  \rebuttal{\textbf{77.16}}\\
    {\ours}, $\alpha = 0$ &  12.76&  \rebuttal{59.09}& \rebuttal{92.87}&  \rebuttal{75.98}\\
    {\ours}, $\alpha = 0$, - buffer&  \textbf{15.59}&  \textbf\rebuttal{60.14}&  \rebuttal{88.76}&  \rebuttal{74.45}\\ \bottomrule
    \end{tabular}
  \end{minipage}
\end{table}

\textbf{Comparison against Previous Safety Alignment Methods.} 
We compare our method with two successful open-source safety alignment methods: Beaver~\citep{dai2024safe} and Llama2~\citep{llama2}. We present statistics on the data size used for RLHF, the need for the red teaming process, and the number of training iterations in \cref{tab:efficiency}. Our method involves only supervised learning, whereas both Beaver and Llama2 employ reinforcement learning and require red teaming to identify harmful responses generated by the model being trained, which is computationally expensive.
Moreover, our approach requires only one iteration of training with {\ours} objective  with just 30K data points, while Beaver and Llama2 conduct multiple iterations of reward learning and reinforcement learning with much larger datasets. 
Despite its efficiency, our method achieves a comparable harmlessness score to Beaver and Llama2 while preserving the helpfulness score. These results indicate strong potential for our method to be applied in the future development of open-source models at a minimal cost.

\subsection{Improve Pre-aligned Models with Red Teaming Data} \label{sec:expredteam}

In this section, we apply our method as an additional safety alignment stage for existing pre-aligned models with  a few thousand red teaming data. We compare our method with DPO~\citep{dpo}, IPO~\citep{ipo}, MORL~\citep{rame2023rewardedsoups} as in \cref{sec:experiment}. 

\textbf{Data Preparation.}    
We first use 9K harmful prompts from the PKU-SafeRLHF dataset~\citep{dai2024safe} and have the Zephyr-7b-beta~\cite{zephyr} model generate two responses for each prompt. We then use the HarmBench-Llama2-13B-Chat~\citep{mazeika2024harmbench} classifier to determine whether the generated responses are harmful. For prompts that result in harmful responses, we use  PairRM~\citep{llm-blender-2023} to rank the responses in terms of helpfulness. This process results in 1.5K harmful prompts, responses, safety labels for each response, and pairwise helpfulness preferences. 

\textbf{Results.}
\cref{tab:redteam} shows the results. Our method improves the harmlessness of the Zephyr-7b-beta model from 60.99\% to 73.90\%, while preserving the helpfulness. The improvement in generative tasks is particularly significant, from 62.94\% to 88.79\%.
The supervised methods, DPO and IPO, can also improve the harmlessness, but the improvement is not as substantial as with our method. When fine-tuning the model with MORL using specific prompts where the model initially struggled as in this experiment, the performance gain is still marginal, though larger than when using general data, as in \cref{tab:comparison}. This aligns with the observation that using RL methods to improve safety requires a large amount of model-specific data, high-quality labels, and a reward model specifically trained on these data to provide distinguishable scores.
In contrast, {\ours} achieves similar goals without the need for large amounts of helpfulness data mixed with red teaming data. Moreover, our overall pipeline of this experiment is efficient and automatic, requiring no human annotation. 
These results strongly indicate that our method can be effectively used in an additional safety alignment stage for existing chat models to improve harmlessness at minimal cost. Full results are in \cref{appen:fullresults}.

\subsection{Ablations}

We validate the technical design, especially the shifting parameter $\alpha$ and the buffer training in \cref{tab:ablation}.
In the {\ours} \(\alpha = 0\) experiment, we set the shift parameter \(\alpha\) to 0. The results indicate that, as illustrated in \cref{sec:illustrative}, the shift parameter \(\alpha\) is useful in distinguishing unsafe data, and thus improves performance on generative tasks in harmlessness slightly.
In the {\ours} - w/o buffer experiment, we do not balance examples from the safety dataset and the helpful dataset, but instead mix the two datasets and randomly sample data from them. The lower harmlessness performance  provides the evidence that buffered training helps mitigate the tension between helpfulness and harmlessness. Full results  and detailed ablation are provided in \cref{appen:fullresults} and  \cref{appen-sec:ablation}.

Other hyper-parameters, like $ \tau, B_1$, are set based on either our theoretical understanding or the past work~\citep{gpo}, and  the fine-tuning strategy is orthogonal to the algorithm, while we further include their ablation 
in \cref{appen-sec:ablation}.

\section{Related Work}

\textbf{Alignment with Diverse Preferences. } Traditional language model alignment methods~\citep{christiano2017deep, stiennon2020learning,  hendrycks2021aligning} typically use a single reward or unified preference model.However, recent work suggests that human preferences are diverse and cannot be adequately represented by a single reward model.  To address this, ~\citet{chakraborty2024maxminrlhf} propose learning a mixture distribution for the reward using the EM algorithm, which they then apply in their MaxMin RLHF approach.  ~\citet{dong2023steerlm, rame2023rewardedsoups,wang2024arithmetic} explore training  multi-objective reward models for the alignment stage. These methods primarily focus on improving reward models for RL based alignment. The most closely related work of supervised alignment methods is by ~\citet{zhou2023beyond}, who, despite advocating for direct policy optimization, still rely on training reward models. In contrast, our approach completely eliminates the two-stage training process and directly integrates multiple preferences into the supervised optimization.

\textbf{Safety Alignment.} Aligning large language models (LLMs) with both helpfulness and harmlessness is a specific case of addressing diverse preferences.  To enhance safety, many researchers conduct additional safety data annotation alongside algorithm design.    ~\citet{llama2} utilizes substantial amounts of human-labeled safety data and combines safety and helpfulness rewards by utilizing the safety reward as a threshold function for the helpfulness reward.  \citet{dai2024safe,ji2024beavertails} engage in red teaming to gather extensive safety data and frame safety alignment as a conditioned Markov Decision Process (MDP) problem. ~\citet{murule} propose a rule-based reward as a complement for the common reward to improve the safety, which, although data-efficient, still requires human annotation and reward learning.  
In contrast, our method is fully automated and efficient, eliminating the need for human intervention in the safety alignment process. 
On the other hand, \citet{huang2024catastrophic} propose generation-aware alignment, which improves the safety over different generation configurations.  With our focus on improving safety under specific configurations, this work can be a strong complement to ours.

\textbf{Safety Evaluation.} Supervised benchmarks, such as OpenLLM~\citep{eval-harness} and BigBench~\citep{srivastava2023beyond}, include datasets related to various aspects of safety, such as toxicity, truthfulness, morality, and social bias.  Adversarial attack research~\citep{zou2023universal} and red teaming efforts~\citep{ji2024beavertails,mazeika2024harmbench} provide valuable toxic prompts to assess if models can generate harmless content in response to these prompts.
To identify if the output content contains harmful information, some studies~\citep{bai2022training,llama2} rely on human annotators, while others employ AI models like GPT-4~\citep{wang2024decodingtrust}. \citet{mazeika2024harmbench} offer fine-tuned Llama2 models to as harmful content classifier, offering an efficient alternative to GPT-4 in model-based evaluation.

\section{Limitations and Discussion}

In this paper, we propose a novel supervised optimization method, \oursfull{} (\ours), to balance the safety and helpfulness during the alignment of LLMs. We theoretically prove that this direct optimization is equivalent to previous multi-objective reinforcement learning that combine safety and helpfulness rewards. With {\ours}, we outperform existing methods in terms of safety and helpfulness in both fine-tuning the pre-trained LLMs and pre-aligned models. 
Our method is highly effective, significantly more computationally efficient, and does not require any human annotation or additional data collection. 

Furthermore, our approach is versatile and does not rely on any specific property of harmlessness itself. This flexibility allows it to be applied to improve various other potentially conflicting objectives in aligning LLMs. To achieve this, we only need characteristic-specific labels for the field-specific dataset. We believe our proposed method can serve as a general framework for the transfer learning of aligned models. However, our method relies on specific label formats for helpfulness and safety may present a limitation when addressing different tasks. Moreover, extending our work to handle more objectives (beyond just two) is also a promising direction for future research.

\section*{Aknowledgement}
This work is supported by a KAUST CRG (URF/1/4648-01-01) and a UKRI grant Turing AI Fellowship (EP/W002981/1).  A. Bibi acknowledges the Google Gemma 2 Academic Award 2024. We also thank the Royal Academy of Engineering.

\bibliography{ref}
\bibliographystyle{iclr2025_conference}

\clearpage
\appendix

\section{Algorithm}
\cref{alg} shows the {\ours} algorithm. As mentioned in \cref{sec:preliminary}, in practice, we refer to datasets related to safety topics, collected through red teaming, as safety datasets. A typical safety dataset will contain a safety label \(\Is(y)\), which is the binary label indicating whether the response \(y\) is harmful, as well as the preference label $\Ih(y\succ y')$ in terms of helpfulness.  If  a certain safety dataset does not provide helpfulness labels, we can use the ranking models, like PairRM~\citep{llm-blender-2023}, as discussed in \cref{sec:expredteam}, to generate the pairwise helpfulness labels.
We refer to datasets designed to improve the helpfulness of the model as helpfulness datasets. A typical helpfulness dataset will contain the helpfulness preference labels \(\Ih(y \succ y')\). Since most helpfulness data undergoes human filtering, the responses are usually safe. Therefore, we assign the safety label \(\Is(y) = 1\) to all responses  in the helpfulness dataset. 

We further require a pre-trained language model $\pi_\text{ref}$, the total number of optimization steps $T$,  the penalty coefficient $\tau$ for the KL divergence term, and the shifting parameter $\alpha$. We also need to specify the layers to be unfrozen for the policy optimization, denoted as $\theta'$. 

At the beginning of the algorithm, we initialize the policy $\pi_\theta$ with the pre-trained language model $\pi_\text{ref}$, and unfreeze the selected layers $\theta'$ (line 1-2). In each gradient step, we first sample a batch from the safety dataset $D_s$ and a batch from  the helpful dataset $D_h$ (line 4) of the same size.   We then compute the loss of both batches according to \cref{eq:wlobjective} (line 6-8). We accumulate the gradients of the loss for the both batches and update the policy $\pi_\theta$ (line 10). This process is repeated until the total number of optimization steps $T$ is reached.

\begin{algorithm}[h]
    \caption{{\ours} Algorithm}\label{alg}
    \begin{algorithmic}[1]
    \Require Safety dataset $D_s = \{ (x,y^{hw},y^{hl}, \Is(y^{hw}), \Is(y^{hl}))\}$ and helpful dataset $D_h= \{ (x,y^{hw},y^{hl})\}$. 
    \Require Total number of optimization steps $T$. Pre-trained language model $\pi_\text{ref}$, and unfrozen layer $\theta'$.  $\tau$, $\alpha$
    \State Initialize $\pi_\theta \leftarrow \pi_\text{ref}$ 
    \State Only unfreeze selected layers $\theta'$
    \While { $t<T$}
    \State Sample batch $B_s \sim D_s$ , $B_h \sim D_h$.
    \For { batch =  $B_s, B_h$}
    \State Compute  $h(y^{hw}, y^{hl})$  with \cref{eq:optimal}
    \State Compute  $g_I$ with \cref{eq:giwl} \Comment{$\Is(y) = 1$ for the helpful dataset.}
    \State Compute and accumulate gradients w.r.t \cref{eq:wlobjective}
    \EndFor
    \State Update $\pi_\theta$. 
    \EndWhile
    \end{algorithmic}
  \end{algorithm}
\section{Proof}

\subsection{Notation}
\begin{table}[h]
    \centering\small\caption{Notations}
    \label{tab:notation}
    \begin{tabular}{ll}\toprule
         Notation& Meaning\\\midrule
         $y, y' \sim \pi(x)$ & Two responses generated independently by the  policy. \\
          $\ph(y\succ y'|x)$& Ground-truth helpfulness preference of $y$ being preferred to $y'$ knowing the context $x$ \\
            $\ps(y|x)$& Ground-truth safety of $y$ knowing the context $x$ \\
            $\Ih(y\succ y'|x)$ & Binary label of helpfulness preference of $y$ being preferred to $y'$ knowing the context $x$ \\
            $\Is(y|x)$ & Binary label of safety of $y$ knowing the context $x$ \\
            $y^w, y^l$ & globally preferred and dispreferred responses knowing the context $x$ \\
            $y^{hw}, y^{hl}$ &  preferred and dispreferred responses in terms of helpfulenss knowing the context $x$ \\
            $E_s$ & Expected safety of a response $y$ given the context $x$ \\
         \bottomrule
    \end{tabular}
\end{table}
\cref{tab:notation} summarizes the notations used in this paper based on~\citet{dpo,ipo}. 
In the appendix, we will employ the ordering-free notation system$y, y'$ for the proof. Specifically, we express the transformation equations from $y^{hw}, y^{hl}$ to $y,y'$ as:
\begin{align*}
    & \Is(y^{hw}|x) = \Ih(y\succ y'|x)\Is(y|x) + \Ih(y'\succ y|x)\Is(y'|x) \\
    &\Is(y^{hl}|x) = \Ih(y\succ y'|x)\Is(y'|x) + \Ih(y'\succ y|x)\Is(y|x)  \\
\end{align*}

For brevity and clarity, we further adopt the notation $y$  to represent $y|x$. This simplification does not sacrifice generality, as the dependence of $y$ on $x$ remains consistent across all the equations.

\subsection{Proof of \cref{thm:unique}}
\label{appensec:proofunique}
We begin by restating  \cref{thm:unique} with the notation system $y,y'$. Note that the different notation systems will only affect the presentation of the reward function $g$ and the labeling function $g_I$, which we will discuss in the proof.

\begin{theorem}
    Let $\tau > 0$ be a real number, $\pi_\theta , \pi_{\text{ref}}$ be two policy.  
    Then  
    \begin{equation}\label{eq_appen:rlhf_optimal}
        \pi^*(y) = \frac{\pi_\text{ref}(y)\exp\left(\tau^{-1} g(y) \right)}{\sum_{s \in \mathcal{S}}\pi_\text{ref}(s)\exp\left(\tau^{-1} g(s) \right)}
    \end{equation}
    is an optimal solution to the optimization problem
    \begin{equation}\label{eq_appen:rlhf_objective}
        \max_{\pi_\theta} \mathbb{E}_{ y \sim \pi_\theta (y)} \big[ g(y) - \tau \text{KL}\left[ \pi_\theta (y ) || \pi_{\text{ref}} (y) \right] \big],
    \end{equation} 
    and $\pi^*(y)$ is the optimal unique solution of 
    \begin{equation}\label{eq_appen:direct_expect_objective}
        \min_{\pi_\theta} \mathbb{E}_{y,y' \sim  \pi_\theta(y)} \left[ h_\pi(y,y') - \frac{g(y) - g(y')}{\tau}\right]^2,
    \end{equation}
    where 
    \begin{equation}\label{eq_appen:h_variable}
        h_\pi(y,y') = \log \left( \frac{\pi_\theta(y) \pi_{\text{ref}}(y')}{\pi_\theta(y') \pi_{\text{ref}}(y)} \right).
    \end{equation}
\end{theorem}

To establish optimal solution, we follow~\citet{ipo} to leverage the following lemma.
\begin{lemma}[~\citet{dpo}, ~\citet{ipo}]\label{lemma:rlhf_optimal}
    Let 
    \[
        \mathcal{L}_\tau (\delta)  = \mathbb{E}_{s \in \delta} [f(s)] - \tau \text{KL}[\delta || \eta] ,
    \]
    where $s \in \mathcal{S}$ and $\mathcal{S}$ is a finite set, $f \in \mathbb{R}^\mathcal{S}$ is a function mapping elements of $\mathcal{S}$ to real numbers, $\delta \in \Delta(\mathcal{S})$ is a probability distribution over $\mathcal{S}$, $\eta \in \Delta(\mathcal{S})$ is a fixed reference distribution, and $\tau \in \mathbb{R}^*_+$ is a strictly positive number. Then the argmax problem   with the regularized criterion
    \[
        \argmax_{\delta \in \Delta(\mathcal{S})} \mathcal{L}_\tau (\delta)
    \]
    has an optimal solution $\delta^*$, where 
    \[
        \delta^*(s) = \frac{ \eta(s)\exp({\tau}^{-1} f(s))}{\sum_{s' \in \mathcal{S}} \eta(s') \exp({\tau}^{-1}  f(s'))}, \,\, \forall s \in \mathcal{S}
    \]   
\end{lemma}
To establish the uniqueness of the solution in \cref{eq_appen:rlhf_optimal} for the optimization problem in \cref{eq_appen:direct_expect_objective}, we leverage the following lemma.

\begin{lemma}[Theorem 2 in ~\citet{ipo}]\label{appen:lmunique}
    Let 
    \begin{equation}\label{eq_appen:ipo_objective}
        \mathcal{L}(\pi_\theta) = \mathbb{E}_{y,y' \sim  \pi_\theta(y)} \left[ h_\pi(y,y') - \frac{g(y) - g(y')}{\tau}\right]^2,
    \end{equation}
    then $\min_{\pi_\theta}  \mathcal{L}(\pi_\theta) $ has a unique optimal solution $\pi^*$ expressed in \cref{eq_appen:rlhf_optimal} 
    , and no other local or global minima exist.
\end{lemma}
\begin{proof}
    Let $J = \supp (\pi) = \{y_1, \dots,y_n\}$, where $n = |J|$, and  $\Pi$ be the set of policies with support set $J$. It is straightforward that $\min_{\pi \in \Pi}  \mathcal{L}(\pi)  = \mathcal{L}(\pi^*) = 0$, thus $\pi^*$ is a global optimal solution. We now prove the uniqueness of this optimal solution by the re-parameterization trick.

     We parameterize $\Pi$ via vectors of logits $s \in \mathbb{R}^J$ of $\pi$, \ie,  $s_i = \log(\pi(y_i))$. Set $\pi_s(y) = \frac{\exp(s_i)}{\sum_{i=1}^n \exp(s_i)}$ for $y = y_i \in J$  and $\pi_s(y)=0$ otherwise. Specially, let $s^*$ be the vector of logits corresponding to $\pi^*$, we have $\pi^* = \pi_{s^*}$.

    We first prove that $s^*$ is the global optimal solution to the optimization problem
    \[
        \mathcal{L}(s) := \mathcal{L}(\pi_s) =  \mathbb{E}_{y,y' \sim \pi_s} \left[ h_{\pi_s}(y,y') - \frac{g(y) - g(y')}{\tau}\right]^2.
    \]
    It is obvious that $\mathcal{L}(s^*) = 0$, thus it is a local minimum. 
    By expanding the square term, we have
    \[
        \begin{split}
            \mathcal{L}(s) & =  \mathbb{E}_{y,y' \sim \pi_s} \left[ \frac{g(y) - g(y')}{\tau} - (s(y) - s(y')) - \log\left( \frac{ \pi_{\text{ref}}(y')}{\pi_{\text{ref}}(y)}\right) \right]^2\\ 
            & = \sum_{y,y' \in J} \pi_s(y) \pi_s(y') \left[ \left( (s(y) - s(y')\right)^2 + C_1 \cdot \left( (s(y) - s(y')\right) + C_2 \right],
        \end{split}
    \]
    where $C_1, C_2$ are two terms independent of $s$. The above equation is a positive semidefinite quadratic form, and hence is convex. Thus, all local minima are global minima. 

    Now we prove that $\pi_{s^*}$  is the unique global minima to $\mathcal{L}(s)$. 
    Since $\pi_s$ is a surjective continuous mapping from $s$ to $\pi$, then every local minima $\pi$ to $\mathcal{L}(\pi)$ corresponds to a set of $s$ that minimizes $\mathcal{L}(s)$.
    The uniquess of $s^*$ will deduce that $\pi^*$ is the unique optimal solution to \cref{eq_appen:direct_expect_objective} and concludes the proof. Consider $s' = s^* + r \cdot \Delta s$, where the only $r$ is the radius and $\Delta s$ is the direction under the polar coordinate. The only direction that  not increase $\mathcal{L}(s' )$ away from $0$ is $e = (\frac{1}{n}, \dots, \frac{1}{n})$ (\citet{boyd2004convex}, Chap. 3). However, we have
    \[
        \pi_{s^* + r \cdot e}(s_i) = \frac{\exp(s_i + r \cdot \frac{1}{n})}{\sum_{i=1}^n \exp(s_i + r \cdot \frac{1}{n})} = \frac{\exp(s_i )}{\sum_{i=1}^n \exp(s_i )} = \pi_{s^*} (s_i), \,\, \forall i \in [n].
    \]
    This indicates that $\pi_{s^*}$ is the unique global minima to $\mathcal{L}(\pi_{s^*})$ and thus $\pi^*$ is the unique optimal solution to \cref{eq_appen:direct_expect_objective}.
\end{proof}

Now we provide the proof of \cref{thm:unique}, most of which follows ~\citet{ipo}.
\begin{proof}
    Let $\mathcal{S}$ be the set of all possible token combinations with fixed token length, then it is finite. Let $f(s) = (p_\text{safe}^*(s) + E_s)(p_\text{help}^*(s \succ \pi)+\frac{1}{2})$, $\delta(s) = \pi_\theta(s)$ and $\eta(s) = \pi_{\text{ref}}(s)$. All the conditions in the \cref{lemma:rlhf_optimal} are satisfied. Thus, \cref{eq_appen:rlhf_optimal} is a solution to  the optimization problem in \cref{eq_appen:rlhf_objective}.

    Now we prove \cref{eq_appen:rlhf_optimal} is also a solution to the optimization problem \cref{eq_appen:direct_expect_objective}. Plug \cref{eq_appen:rlhf_optimal} in the \cref{eq_appen:direct_expect_objective}, we have 
    \[
            h_{\pi^*}(y,y') = \log \left( \frac{\pi^*(y) \pi_{\text{ref}}(y')}{\pi^*(y') \pi_{\text{ref}}(y)} \right)  = \log \left(\frac{\exp\left(\tau^{-1} g(y) \right)}{\exp\left(\tau^{-1} g(y') \right)}\right) = {\tau}^{-1} (g(y) - g(y')),
    \]
    which validates \cref{eq_appen:rlhf_optimal} is a solution to the optimization problem \cref{eq_appen:direct_expect_objective}.

    Finally,  \cref{appen:lmunique} indicates \cref{eq_appen:rlhf_optimal} is the unique solution to  \cref{eq_appen:direct_expect_objective}. This concludes the proof. 
\end{proof}
The above proof holds for any order of $y,y'$ since the equation in \cref{eq_appen:h_variable} is  skew-symmetric, \ie, 
\[
    \left[ h_\pi(y,y') - \frac{g(y) - g(y')}{\tau}\right]^2 = \left[ h_\pi(y',y) - \frac{g(y') - g(y)}{\tau}\right]^2.
\]
This allows us to freely arrange the order of $y,y'$ in \cref{eq_appen:direct_expect_objective} without loss of generality. Therefore, \cref{eq_appen:direct_expect_objective} can be written as 
\[
    \min_{\pi_\theta} \mathbb{E}_{y,y' \sim  \pi_\theta(y)} \left[ h_\pi(y^{hw},y^{hl}) - \frac{g(y^{hw}) - g(h^{hl})}{\tau}\right]^2,  
\]
where 
\[
y^{hw} = \begin{cases}
    y & \text{if } \Ih(y\succ y'|x) = 1,\\
    y' & \text{otherwise},\end{cases}
\]
and 
\[
y^{hl} = \begin{cases}
    y' & \text{if } \Ih(y\succ y'|x) = 1,\\
    y & \text{otherwise}.\end{cases}
\]
With this reordering, the theorem reduces to \cref{thm:unique}

\subsection{Proof of \cref{thm:equivalence}}
\label{appensec:proofequivalence}
In this section, we prove the \cref{thm:equivalence}. We begin by rewriting the formula in \cref{eq:gi} into a function of $y, y'$.
\begin{equation}\small
    \begin{split}
        &g_I(y,y') =  B_3\Big( B_1 \big(\Is(y)\Ih(y \succ y') + \Is(y')\Ih(y'\succ y)  \big) \\
        &-  \big(\Is(y)\Ih(y' \succ y ) + \Is(y')\Ih(y\succ y')\big) +B_2 \Big) \cdot  \Big(2\Ih(y \succ y' ) -1 \Big),
    \end{split}
\end{equation}
Here, $\Ih(y \succ y' )$ determines whether $y$ is the win response or lose response. In other words, 
\[
\Is(y^{hw}) = \Is(y)\Ih(y \succ y' ) + \Is(y')\Ih(y'\succ y) ,
\]
and the same applies to $\Is(y^{hl})$.
To enable the reordering of the variables,  we further multiply the formula by $2\Ih(y \succ y' ) -1$, since $h_\pi(y,y') = -h_\pi(y',y)$
By organizing the terms, we have
\[
\begin{split}
    g_I(y,y') = &(B_1B_3-B_3) \Ih(y \succ y')\Is(y) + (B_1B_3-B_3) \Ih(y \succ y')\Is(y') \\
    &    -B_1B_3\Is(y') + B_3\Is(y)  + 2B_2B_3\Ih(y \succ y') -B_2B_3
\end{split}
\]
We first establish the equivalence of the two optimization problems in \cref{eq_appen:expect_optm} and \cref{eq_appen:direct_optm} under the specific choice of constants, and then provide the general relation of constants for the equivalence.

Here, we use the following constants:
$$A_1=E_s, A_2=\frac{1}{2}, B_1=3, B_2=0, B_3=\frac{1}{2}.$$ 

\begin{theorem} \label{thm_appen:equivalence}
    The optimization problem 
    \begin{equation}\label{eq_appen:expect_optm}
        \min_{\pi_\theta} \mathbb{E}_{x \sim \rho,y,y' \sim  \pi_\theta(y)} \Big[ h_\pi(y,y') -\frac{g\big(\ps(y),\ph(y)\big) - g\big(\ps(y'),\ph(y')\big) }{\tau}\Big]^2,
    \end{equation}
    where $g(y) = (p_\text{safe}^*(y) + E_s)(p_\text{help}^*(y \succ \pi)+\frac{1}{2})$,
    is equivalent to the optimization problem
    \begin{equation}\label{eq_appen:direct_optm}
        \min_{\pi_\theta} \mathbb{E}_{x \sim \rho, y,y' \sim  \pi_\theta(y), I\sim \text{Bernoulli}} \bigg[ \Big( h_\pi(y,y')  - 
        \frac{g_I(y,y')}{\tau}\Big)^2 \bigg], 
    \end{equation}
    where $$g_I (y,y') =  \Ih(y \succ y' )\Is(y) + \Ih (y \succ y')\Is(y') + \frac{1}{2}\Is(y) - \frac{3}{2}\Is(y')$$
\end{theorem}
Here, $I\sim \text{Bernoulli}$ denotes the Bernoulli variables  $\Is(y)$ and $\Is(y')$.

\begin{proof}
The two minimization problems are both over $\pi_\theta$, so we only need to focus on the terms that involve  $\pi_\theta$. Specifically,  the first term and the cross term after expanding the square expression in the two minimization problems. The first term is the same. Here we prove the cross term is also the same.

Let $\pi_y = \log(\pi(y))$, $\pi_y^R = \log(\pi_\text{ref}(y))$, then we can write 
\[
h_{\pi} (y,y') = \pi_{y} - \pi_{y'} + \pi^{R}_{y'} - \pi^{R}_{y}
\]
Let $p_h(y) = p_\text{help}^*(y \succ \pi )$ and $p_s(y) =  p_\text{safe}^*(y) $. 
The cross term of \cref{eq_appen:expect_optm} can be written as 
\begin{equation}\label{eq:proofrhs}
\begin{split}
&\mathbb{E}_{x \sim \rho, y,y' \sim  \pi} \left[  h_\pi(y,y') \left(g\big(\ps(y),\ph(y \succ \pi )\big) - g\big(\ps(y'),\ph(y' \succ \pi )\big) \right) \right] \\
    = & \mathbb{E}_{x \sim \rho, y,y' \sim  \pi} \left[( \pi_{y} - \pi_{y'} + \pi^{R}_{y'} - \pi^{R}_{y})  \left( g\big(p_s(y),p_h(y)) - g(p_s(y'),p_h(y')\big) \right) \right]  \\
    =& \mathbb{E}_{x \sim \rho, y \sim  \pi} \left[( \pi_{y}  - \pi^{R}_{y}) \left(g(p_s(y),p_h(y)) - \mathbb{E}_{y' \sim  \pi}\big[g\big(p_s(y'),p_h(y')\big)\big] \right) \right]\\
    + & \mathbb{E}_{x \sim \rho, y' \sim  \pi} \left[(  - \pi_{y'} + \pi^{R}_{y'} ) \left(\mathbb{E}_{y \sim  \pi}\big[g\big(p_s(y),p_h(y)\big)\big] - g\big(p_s(y'),p_h(y')\big) \right) \right] 
\end{split}
\end{equation}
The third equality is by the independence of $y$  and $y'$. By the  change of notation, the second term of the last line can be written as 
\begin{equation}\label{eq:changeofvariable}
    \begin{split}
        & \mathbb{E}_{x \sim \rho, y' \sim  \pi} \left[(  - \pi_{y'} + \pi^{R}_{y'} )\left( \mathbb{E}_{y \sim  \pi}\big[g(p_s(y),p_h(y))\big] - g(p_s(y'),p_h(y')) \right) \right] \\
    = & \mathbb{E}_{x \sim \rho, y \sim  \pi} \left[(  - \pi_{y} + \pi^{R}_{y} ) \left(\mathbb{E}_{y'\sim  \pi}\big[g(p_s(y'),p_h(y'))\big] - g(p_s(y),p_h(y)) \right) \right]
    \end{split}
\end{equation}
 Then~\cref{eq:proofrhs} can be written as 
\begin{equation}\label{eq:generalrhs}
     (\ref{eq:proofrhs})    =  \mathbb{E}_{x \sim \rho, y \sim  \pi} \left[( \pi_{y}  - \pi^{R}_{y})\cdot 2\left(g(p_s(y),p_h(y)) - \mathbb{E}_{y' \sim  \pi}\big[g(p_s(y'),p_h(y'))\big] \right) \right] 
\end{equation}

Now we plug in $g(p_s(y),p_h(y))= (p_s(y) + E_s)(p_h(y)+\frac{1}{2})$ and use the fact $\mathbb{E}_{y' \sim  \pi}[p_h(y' \succ \pi)] = \frac{1}{2}$.  
\cref{eq:generalrhs} can be expanded as
\begin{equation}\label{eq:finaexpect}
    \begin{split}
   (\ref{eq:proofrhs}) = & \mathbb{E}_{x \sim \rho, y \sim  \pi} \left[( \pi_{y}  - \pi^{R}_{y})\cdot 2\left((p_s(y) + E_s)(p_h(y)+\frac{1}{2}) - \mathbb{E}_{y' \sim  \pi}\big[(p_s(y') + E_s)(p_h(y')+\frac{1}{2})\big] \right) \right] \\
   = & \mathbb{E}_{x \sim \rho, y \sim  \pi} \left[( \pi_{y}  - \pi^{R}_{y})\cdot 2\left((p_s(y) + E_s)(p_h(y)+\frac{1}{2}) -2E_s \right) \right] \\
   = & \mathbb{E}_{x \sim \rho, y \sim  \pi} \left[( \pi_{y}  - \pi^{R}_{y})\cdot\left( 2p_s(y)p_h(y) +2E_sp_h(y)+ p_s(y) - 3E_s\right) 
  \right]
\end{split}
\end{equation}

The cross term of  \cref{eq_appen:direct_optm} can be written as 
\begin{equation}\label{eq:prooflhs}
\begin{split}
    & \mathbb{E}_{x \sim \rho, y,y' \sim  \pi} \mathbb{E}_{I\sim \text{Bernoulli} } \left[ h_\pi(y,y') g_I(y,y')\right] \\ 
= &\mathbb{E}_{x \sim \rho, y,y' \sim  \pi} \mathbb{E}_{I\sim \text{Bernoulli} } \left[( \pi_{y} - \pi_{y'} + \pi^{R}_{y'} - \pi^{R}_{y})g_I(y,y')\right] \\
\end{split}
\end{equation}
Now we plug in $g_I =  \Ih(y \succ y' )\Is(y) + \Ih (y \succ y')\Is(y') + \frac{1}{2}\Is(y) - \frac{3}{2}\Is(y') $, 
\[
\begin{split}
    (\ref{eq:prooflhs}) = & 
        \mathbb{E}_{x \sim \rho, y,y' \sim  \pi} \mathbb{E}_{I\sim \text{Bernoulli} } \Big[( \pi_{y} - \pi_{y'} + \pi^{R}_{y'} - \pi^{R}_{y})  \big( \Ih(y \succ y' )\Is(y) \\
        & \hspace{5cm}  + \Ih (y \succ y')\Is(y')+ \frac{1}{2}\Is(y) - \frac{3}{2}\Is(y') \big)\Big]\\
    = &  \mathbb{E}_{x \sim \rho, y,y' \sim  \pi} \mathbb{E}_{I\sim \text{Bernoulli} } \Big[( \pi_{y}  - \pi^{R}_{y})\big( \Ih(y \succ y' )\Is(y) \\
    &  \hspace{5cm}+ \Ih (y \succ y')\Is(y') + \frac{1}{2}\Is(y) - \frac{3}{2}\Is(y') \big)\Big] \\
     &  +  \mathbb{E}_{x \sim \rho, y,y' \sim  \pi} \mathbb{E}_{I\sim \text{Bernoulli} } \Big[( - \pi_{y'} + \pi^{R}_{y'} )\big( \Ih(y \succ y' )\Is(y)  \\
     & \hspace{5cm}+ \Ih (y \succ y')\Is(y') + \frac{1}{2}\Is(y) - \frac{3}{2}\Is(y') \big)\Big] \\
\end{split}
\]
With the similar change of notation as~\cref{eq:changeofvariable}, as well as the fact that $1-\Ih(y \succ y' ) = \Ih(y' \succ y )$, the last line can be written as 
\[
\begin{split}
    &  \mathbb{E}_{x \sim \rho, y,y' \sim  \pi} \mathbb{E}_{I\sim \text{Bernoulli} } \Big[( - \pi_{y'} + \pi^{R}_{y'} )\big( \Ih(y \succ y' )\Is(y) \\
    & \hspace{5cm} + \Ih (y \succ y')\Is(y') + \frac{1}{2}\Is(y) - \frac{3}{2}\Is(y') \big)\Big] \\ 
    = & \mathbb{E}_{x \sim \rho, y,y' \sim  \pi} \mathbb{E}_{I\sim \text{Bernoulli} } \Big[( - \pi_{y} + \pi^{R}_{y} )\big( \Ih(y' \succ y )\Is(y') \\
    & \hspace{5cm} + \Ih (y' \succ y)\Is(y) + \frac{1}{2}\Is(y') - \frac{3}{2}\Is(y) \big)\Big] \\
    = & \mathbb{E}_{x \sim \rho, y,y' \sim  \pi} \mathbb{E}_{I\sim \text{Bernoulli} } \Big[( - \pi_{y} + \pi^{R}_{y} )\big( (1-\Ih(y \succ y' ))\Is(y') \\
    & \hspace{5cm} +  (1-\Ih(y \succ y' ))\Is(y) + \frac{1}{2}\Is(y') - \frac{3}{2}\Is(y)  \big)\Big]
\end{split}
\]
Then we further expand~\cref{eq:prooflhs} as 
\begin{equation}\label{eq_appen:expand_direct}
    \begin{split}
    (\ref{eq:prooflhs}) = & \mathbb{E}_{x \sim \rho, y,y' \sim  \pi} \mathbb{E}_{I\sim \text{Bernoulli} } \Big[( \pi_{y}  - \pi^{R}_{y})\big( \Ih(y \succ y' )\Is(y) \\
    & \hspace{5cm}  + \Ih (y \succ y')\Is(y') + \frac{1}{2}\Is(y) - \frac{3}{2}\Is(y') \big)\Big] \\
    &+   \mathbb{E}_{x \sim \rho, y,y' \sim  \pi} \mathbb{E}_{I\sim \text{Bernoulli} } \Big[( - \pi_{y} + \pi^{R}_{y} )\big( (1-\Ih(y \succ y' ))\Is(y') \\
    &  \hspace{5cm}  +  (1-\Ih(y \succ y' ))\Is(y) + \frac{1}{2}\Is(y') - \frac{3}{2}\Is(y) \big)\Big]\\
    =  & \mathbb{E}_{x \sim \rho, y,y' \sim  \pi} \mathbb{E}_{I\sim \text{Bernoulli} } \Big[( \pi_{y}  - \pi^{R}_{y}) \big(2\Ih(y \succ y' )\Is(y) \\
    & \hspace{5cm} + 2  \Ih (y \succ y')\Is(y') + \Is(y) -3\Is(y')
    \big)\Big] 
\end{split}
\end{equation}
Taking the expectation over $y'$ and the Bernoulli variables, we have 
\begin{equation}\label{eq_appen:final_direct}
    (\ref{eq:prooflhs}) = \mathbb{E}_{x \sim \rho, y \sim  \pi} \Big[( \pi_{y}  - \pi^{R}_{y}) \big( 2p_h(y)p_s(y) + 2E_sp_h(y) + p_s(y) - 3E_s\big)\Big] 
\end{equation}
This equation is the same as~\cref{eq:finaexpect}, which ends the proof that  \cref{eq_appen:expect_optm} and \cref{eq_appen:direct_optm} are equivalent!
\end{proof}

As discussed in  \cref{appensec:proofunique}, we can freely change the order of $y$ and $y'$ in \cref{eq_appen:expect_optm} and \cref{eq_appen:direct_optm}. Thus, the proof of \cref{thm_appen:equivalence} also applies to \cref{thm:equivalence}.

\subsection{Relation of the Constants}
\label{appensec:proofconstants}

In this section, we derive a more general form of \cref{thm_appen:equivalence}, where,  with specific relations between the constants in $g$ and $g_I$, the optimization problem in \cref{eq_appen:expect_optm} is equivalent to the optimization problem in \cref{eq_appen:direct_optm}.

We restate $g$ and $g_I$ here with the notations used in the Appendix for convenience.
\[
    g = (p_s(y )  + A_1)(p_h(y)  + A_2),
\]
and 
\[
\begin{split}
    g_I(y,y') = &(B_1B_3-B_3) \Ih(y \succ y')\Is(y) + (B_1B_3-B_3) \Ih(y \succ y')\Is(y') \\
    &    -B_1B_3\Is(y') + B_3\Is(y)  + 2B_2B_3\Ih(y \succ y') -B_2B_3
\end{split}
\]
As discussed in the proof of \cref{thm_appen:equivalence}, we only need to find the  relationship such that the cross terms of the two optimization problems are identical. We first expand the cross term of the optimization problem in \cref{eq_appen:expect_optm}. As in \cref{eq:generalrhs}, it can be written as
\begin{equation}\label{eq_appen:constant_expect}
    (\ref{eq:proofrhs})    =  \mathbb{E}_{x \sim \rho, y \sim  \pi} \left[( \pi_{y}  - \pi^{R}_{y})\cdot 2\left(g(p_s(y),p_h(y)) - \mathbb{E}_{y' \sim  \pi}\big[g(p_s(y'),p_h(y'))\big] \right) \right] 
\end{equation}

Using the same strategy of obtaining \cref{eq_appen:expand_direct},
 we have 
\begin{multline}\label{eq_appen:constant_direct}
    (\ref{eq:prooflhs}) = \mathbb{E}_{x \sim \rho, y \sim  \pi} \Big[( \pi_{y}  - \pi^{R}_{y}) \big( 2B_3(B_1-1) p_s(y)p_h(y) \\ + 2B_3((B_1-1)E_s + 2B_2)p_h(y)   +  2B_3p_s(y)  - 2B_1B_3E_s  -2 B_2B_3\big)\Big] 
\end{multline}
Aligning the coefficients of each term in \cref{eq_appen:constant_expect} and \cref{eq_appen:constant_direct}, we derive the following set of equations: 
\begin{equation}\label{eq_appen: constant_equality}
   \begin{split}
     B_3(B_1 - 1)&= 1,\\
     E_s + 2B_3B_3 &= A_1, \\
     B_3 &= A_2 .
\end{split} 
\end{equation}
Solving these equations gives us the specific forms of  $g$ and $g_I$. Here $B_2$ is a shifting value that we define to align with our intuition. $B_3$ is a scaling factor that is related to the penalty $\tau$.

\subsection{Discussion of the Property of $g_I$}
\label{appensec:discussproperty}

In this section, we discuss the two beneficial properties of \( g_I \) that we proposed in \cref{sec:equivalence}.

\textbf{Skew-Symmetric Property.}
First, we examine the skew-symmetric property of \( g_I \).  When combined with the skew-symmetric property of \( h \), this implies:
\[
\left( h_\pi(y,y') - \tau^{-1} g_I(y,y') \right)^2 = \left( h_\pi(y',y) - \tau^{-1} g_I(y',y) \right)^2.
\]
This means that for the same data point, regardless of the order of \( y \) and \( y' \), we are always driving \( h_\pi(y,y') \) to the same value. In contrast, in IPO~\citep{ipo}, different orders  will push \( h_\pi(y,y') \) to different values, i.e., they form two different optimization problems:
\[
(h_\pi(y,y') - \tau^{-1} g_I(y,y'))^2 \quad \text{and} \quad (h_\pi(y',y))^2.
\]
Their final optimization problem, \((h_\pi(y,y') - \frac{1}{2} \tau^{-1} g_I(y,y'))^2\), tries to find a middle point of \( h \) that optimizes both. However, this point is neither the optimal solution of the first problem nor the second problem.

\textbf{Shifting Property.}
Second, we discuss the shifting properties of \( g_I \). Since \cref{thm:equivalence} holds based on the equality of \cref{eq_appen:final_direct} and \cref{eq:finaexpect}, and all the operations to derive these two equations are valid under linear transformations of \(\ps, \ph\) and \(\Is, \Ih\), respectively. It implies that \cref{thm:equivalence} also holds under the same linear transformations of \(\ps, \ph\) and \(\Is, \Ih\).

This property is useful when we want to manually design the values of \( g_I \), as shown in \cref{fig:pairwiserank}.

\section{Experiment}

\subsection{Details of Harmlessness Benchmark}\label{appen:benchmark}

The following are the details of the datasets used in the benchmark:
\begin{itemize}
    \item \textbf{CrowS-Pairs}: A challenge set for evaluating the tendency of language models (LMs) to generate biased outputs. We evaluate the English subset and \href{https://github.com/EleutherAI/lm-evaluation-harness/blob/main/lm_eval/tasks/crows_pairs/README.md}{implementation by LM Evaluation Harness} and report the Percentage of Stereotype metric. 
    \item \textbf{Bias Benchmark for QA (BBQ)}: Measures social biases in the output of multiple-choice question-answering models.  We use the Lite version and \href{https://github.com/google/BIG-bench/tree/main/bigbench/benchmark_tasks/bbq_lite}{implementation by BigBench} and report the Accuracy metric.
    \item \textbf{WinoGrande}: A collection of 44K problems inspired by the Winograd Schema Challenge, adjusted to improve scale and robustness against dataset-specific biases. We use the \href{https://github.com/EleutherAI/lm-evaluation-harness/tree/main/lm_eval/tasks/winogrande}{implementation by LM Evaluation Harness} and report the Accuracy metric.
    \item \textbf{ETHICS}: A benchmark spanning concepts in justice, well-being, duties, virtues, and commonsense morality.  We use the \href{https://github.com/EleutherAI/lm-evaluation-harness/tree/main/lm_eval/tasks/winogrande}{implementation by LM Evaluation Harness} and report the Accuracy for all sub-tasks.
    \item  \textbf{BigBench Moral Permissibility}: Tests whether ultra-large language models can read a short story where a moral scenario is presented and answer the question, "Is it morally permissible to do X?" in a manner similar to humans. We use the \href{https://github.com/google/BIG-bench/blob/main/bigbench/benchmark_tasks/moral_permissibility/README.md}{implementation by BigBench} and report the Accuracy metric.
    \item \textbf{BigBench Simple Ethics Questions}: Assesses a language model's responses to hypothetical, consequential, political, and social questions. We use the \href{https://github.com/google/BIG-bench/tree/main/bigbench/benchmark_tasks/simple_ethical_questions}{implementation by BigBench} and report the Accuracy metric.
    \item \textbf{ToxicGen}: A benchmark for evaluating the ability of language models to classify input text as either hateful or not hateful.  We use the \href{https://github.com/EleutherAI/lm-evaluation-harness/tree/main/lm_eval/tasks/toxigen}{implementation by LM Evaluation Harness} and report the Accuracy metric.
    \item \textbf{BigBench HHH Alignment}: Evaluates language models on alignment, pragmatically broken down into the categories of helpfulness, honesty/accuracy, harmlessness, and other aspects.  We use the \href{https://github.com/google/BIG-bench/tree/main/bigbench/benchmark_tasks/hhh_alignment}{implementation by BigBench} and report the Accuracy metric.
    \item \textbf{AdvBench} contains harmful prompts. We use the prompts provided \href{https://github.com/llm-attacks/llm-attacks/blob/main/data/advbench/harmful_behaviors.csv}{here} and generation implementation by  LM Evaluation Harness. We report the percentage of harmless responses measured by \href{https://huggingface.co/cais/HarmBench-Llama-2-13b-cls}{HarmBench-Llama-2-13b-cls} model.  
    \item \textbf{RealToxicityPrompts}: A benchmark for evaluating the ability of language models to continue a prompt in a non-toxic way.  We use the \href{https://github.com/EleutherAI/lm-evaluation-harness/tree/main/lm_eval/tasks/toxigen}{implementation by LM Evaluation Harness} report the percentage of harmless responses measured by \href{https://huggingface.co/cais/HarmBench-Llama-2-13b-cls}{HarmBench-Llama-2-13b-cls} model. 
    \item \textbf{ALERT}: A benchmark to assess the safety of LLMs through red teaming methodologies.   We use the prompts provided \href{https://huggingface.co/datasets/Babelscape/ALERT}{here} and generation implementation by  LM Evaluation Harness. We report the percentage of harmless responses measured by \href{https://huggingface.co/cais/HarmBench-Llama-2-13b-cls}{HarmBench-Llama-2-13b-cls} model. 
    \item \textbf{ALERT Adversarial}: A benchmark to assess the safety of LLMs through red teaming methodologies with adversarial prompts.  We use the prompts provided \href{https://huggingface.co/datasets/Babelscape/ALERT/viewer/alert_adversarial}{here} and generation implementation by  LM Evaluation Harness. We report the percentage of harmless responses measured by \href{https://huggingface.co/cais/HarmBench-Llama-2-13b-cls}{HarmBench-Llama-2-13b-cls} model.  
    \item \textbf{AlpacaEval} Based on the AlpacaFarm evaluation set, which tests the ability of models to follow general user instructions. We employ the \href{https://github.com/tatsu-lab/alpaca_eval}{official implementation} report the LC Win Rate. 
\end{itemize}

\subsection{Details of Baselines}\label{appen:baselines}
The following are the details of the methods that align LLMs for multiple objectives. 
\begin{itemize}
    \item \textbf{Llama2}~\citep{llama2} trains the safety reward $r_\text{safe}$ and the helpfulness reward $r_\text{help}$ separately, and   defines the global reward $g$ as a combination of these rewards, \ie,
    \begin{align*}
        & \Tilde{g} (y|x) = \begin{cases}
            r_\text{safe}(y|x) \text{ if }  \text{IS\_SAFETY}(x), \text{ or } r_\text{safe}(y|x)< 0.15,\\
            r_\text{help}(y|x) \text{ otherwise},
            \end{cases}\\
        & g(y|x) = \text{WHITEN}(\text{LOGIT}(\Tilde{g} (y|x))).
    \end{align*}
    Here IS\_SAFETY$(x)$  indicates whether prompts are tagged as unsafe in their dataset, and the 0.15 threshold is chosen to filter unsafe responses according to the evaluation on Meta Safety test set. Whitening the final linear scores is to increase stability. The global reward is used in the RLHF objective in \cref{eq:multiobjective}.
    \item \textbf{Beaver}~\citep{dai2024safe} trains the safety reward $r_\text{safe}$ and the helpfulness reward $r_\text{help}$ separately, and   defines the final RLHF objective as the dual optimization problem of the conditional RLHF, obtained by Lagrangian dual transformation, \ie, 
    \[
    \min_\theta \max_{\lambda \geq 0} \mathbb{E}_{x\sim D, y\sim\pi_\theta} \left[ -r_\text{help}(y|x) + \lambda \left( r_\text{safe}(y|x) +d\right) \right],
    \]
    where $\lambda \geq 0$ is the Lagrange multiplier.  In practice, the model parameter $\theta$ and the Lagrange multiplier $\lambda$ are updated iteratively.
    \item \textbf{RBR}~\citep{murule} requires separate reward models, $r_{\phi_1}, \dots, r_{\phi_k} $, for each objective,  and propose to learn the weight for each objective, \ie,
    \item \[
    g(y|x) = \sum_{i=1}^k \lambda_i r_i(y|x),
    \]
    where $\lambda_i$ are learnable parameters. The global reward is used in the RLHF objective in \cref{eq:multiobjective}.
    \item \textbf{SteerLM}~\citep{dong2023steerlm} trains models to generate response according to a specific reward vector $r = (r_1, r_2, r_3,\dots,r_k)$. They first train a model to predict the score for each objective  in a dataset.  Supervised fine-tuning is performed to maximize the probability of generating responses conditioned on the reward vector and the prompt, \ie,
    \[
    \max_\theta \mathbb{E}_{(x,y,r)\sim D}  \log p_\theta(y|x,r).
    \]
    \item \textbf{MORL}~\citep{rame2023rewardedsoups} trains reward models for each objective separately, and defines the global reward $g$ as a combination of rewards, \ie, 
    \[
    g(y|x) = \sum_{i=1}^k \lambda_i r_i(y|x),
    \]
    The global reward is used in the RLHF objective in \cref{eq:multiobjective}.
    \item \textbf{ArmoRM}~\citep{wang2024arithmetic} applies the same training strategy as MORL, but uses a single publicly available reward model, ArmoRM-Llama3-8B-v0.1~\citep{ArmoRM}, to provide the reward scores for all objectives. 
    \item \textbf{MODPO}~\citep{zhou2023beyond} trains margin reward models $r_i, i=1,\dots,k$ for each objective separately, and performs supervised fine-tuning with the objective, 
    \begin{multline*}
    \max \mathbb{E}_{(x,y^w,y^l)\sim D} \\ \log \sigma\left( \frac{1}{\omega_k} \left( \tau \log \frac{\pi_\theta(y^w|x)}{\pi_{\text{ref}}(y^w|x)} - \tau \log \frac{\pi_\theta(y^l|x)}{\pi_{\text{ref}}(y^l|x)}  - {\omega}_{-k}^T (r_{-k}(x,y^w) - r_{-k}(x,y^l) )\right) \right),
 \end{multline*}
    where $\omega_k$ is the weight for the objective $k$, $\omega_{-k}$ is the weight vector for all other objectives, and $r_{-k}$ is the reward vector for all other objectives than $k$. This fine-tuning is performed  for each objective.
    \item \textbf{MinMaxRLHF}~\citep{chakraborty2024maxminrlhf} addresses the scenario where different annotators $h$ may have preferences for different objectives. The algorithm uses the EM algorithm to learn the distribution of rewards for multiple objectives.   In the E step, they find the certain objective $i$ that each human annotator $h$ relates to, \ie,
    \[
    \mathcal{I}_h =  \argmax_{i} \Pi_{x,y,y',h}\frac{\exp(r_{\phi_i} (x,y))}{\exp(r_{\phi_i} (x,y)) + \exp(r_{\phi_i} (x,y'))},
    \]
    where $r_{\phi_i}$ is the reward model for the objective $i$. In the M step, each reward model $i$ is updated by the reward learning objective in \cref{eq:rewardlearning} with the data assigned to objective  $i$, \ie, the dataset is $D_i = \{(x,y,y',h),    \mathcal{I}_h=i \}$. In the RLHF stage, they maximize the minimum reward of all reward scores, \ie, 
    \[
    \mathbb{E}_{x\sim D, y\sim\pi_\theta} \big[ \min_i r_{\phi_i} (x,y) - \tau \text{KL} \left[ \pi_\theta (y | x) || \pi_{\text{ref}} (y | x) \right] \big].
    \]
\end{itemize}
Among these methods, MODPO is highly inefficient since it requires separate RLHF for each objective. Other methods typically use a linear combination of reward scores for multiple objectives or one reward as a threshold for others.  For the combination of thresholding, the global function can be approximated by the multiplication of rewards for each objective when the reward scores are on the same scale.  Maximizing the multiplication of rewards has the same effect as maximizing the minimum reward. Therefore, we hypothesize that the global reward should be a bilinear combination of the reward scores as in \cref{eq:desiredreward}. As such, in the experiment section, we  select MORL as a representative for this line of approach. 

\subsection{Additional Ablation Studies} \label{appen-sec:ablation}
We include additional ablation studies of other hyper-parameters in our algorithms. Full results are in \cref{tab:fullresults}.

\paragraph{LoRA finetuning. }
We follow the setting of \cref{sec:experiment} and conduct additional experiments to compare our method and the best performed baseline, IPO, with LoRA fine-tuning. During the training, we apply the same training hyper-parameters to both algorithms, like learning rate, training epochs, beta, and so on. The results are in \cref{appen-tab:lorafinetuning}. Results show that BFPO consistently outperformed the baselines when training with LoRA. However, we observed that LoRA training required additional hyperparameter tuning, which posed challenges due to the limited time. Consequently, both methods achieved lower overall performance and worse balance compared to selective fine-tuning.

\begin{table}[h]
    \centering\small
    \begin{tabular}{lcccccccc} \toprule
 & \multicolumn{4}{c}{LoRA Fine-tuning}& \multicolumn{4}{c}{Selective Fine-tuning}\\
       \cmidrule(lr){2-5} \cmidrule(lr){6-9}
        & Helpfulness& \multicolumn{3}{c}{Harmlessness} & Helpfulness& \multicolumn{3}{c}{Harmlessness}\\\cmidrule(lr){2-2} \cmidrule(lr){3-5}\cmidrule(lr){6-6} \cmidrule(lr){7-9}
        & Alpaca& Disc.& Gen.& Savg. & Alpaca& Disc.&Gen.&Savg. \\
        \midrule
        IPO & 6.14 & 58.05 & 93.97 & 76.1  & 13.15& 58.41&  89.76&  74.09\\
        BFPO & 7.77 & 64.36 & 94.73 & 79.54  & 3.33& 59.09& 95.24&77.16\\
        \bottomrule
    \end{tabular}
    \caption{Comparison of Helpfulness and Harmlessness Metrics with LoRA finetuning}
    \label{appen-tab:lorafinetuning}
\end{table}

\begin{table}[t]\centering
  \begin{minipage}[t]{0.49\textwidth}\small\centering
      \centering
    \caption{Ablation Study of $\alpha$ in \cref{eq:wlobjective}}
    \label{appen-tab:alpha}
    \begin{tabular}{cccccc}
        \toprule
       \multirow{2}{*}{$\alpha$} & Helpfulness & \multicolumn{3}{c}{Harmlessness} \\ \cmidrule(lr){2-2} \cmidrule(lr){3-5} 
        & Alpaca  & Disc. & Gen. & Savg. \\ 
        \midrule
        0.1 & 13.61 & 59.81 & 87.39 & 73.60 \\
        0.3 & 14.06 & 60.31 & 91.73 & 76.02 \\
        0.5 & 13.33 & 59.09 & 95.24 & 77.16 \\
        0.7 & 9.01  & 57.34 & 96.28 & 76.81 \\
        0.9 & 7.21  & 56.44 & 96.66 & 76.55 \\
        \bottomrule
    \end{tabular}
  \end{minipage}\hfill
  \begin{minipage}[t]{0.49\textwidth}\small\centering
      \caption{Ablation Study of $\tau$ in \cref{eq:wlobjective}}
    \label{appen-tab:tau}
    \begin{tabular}{ccccc}
         \toprule
       \multirow{2}{*}{$\tau$} & Helpfulness & \multicolumn{3}{c}{Harmlessness} \\ \cmidrule(lr){2-2} \cmidrule(lr){3-5} 
        & Alpaca  & Disc. & Gen. & Savg. \\ 
        \midrule
        0.01 & 13.33 & 59.09 & 95.24 & 77.16 \\
        0.1  & 6.4   & 55.44 & 81.45 & 68.44 \\
        0.5  & 6.53  & 54.01 & 78.14 & 66.07 \\
        1.0  & 6.74  & 54.10 & 77.52 & 65.81 \\
        \bottomrule
    \end{tabular}
  \end{minipage}
\end{table}

\paragraph{The hyperparameter $\alpha$. } The hyperparameter 
 controls the label values (represent the difference of the preference of a pair of response) of the four cases in \cref{fig:pairwiserank}. To ensure the desired behavior, that helpful-safe responses are preferred over helpless-safe ones (case 2 in \cref{fig:pairwiserank} yields a positive value) and that helpful-unsafe responses are not preferred over helpless-unsafe ones (Case 3 in \cref{fig:pairwiserank} yields a negative value), we constrain $\alpha \in (0,1)$.
When $\alpha = 0.5$, the label values for the four cases are $1, 0.5, -0.5, -1$, where the absolute label values are symmetric for positive and negative pairs. As $\alpha$ increases, the absolute label values in case 1,2 in \cref{fig:pairwiserank} decrease, and the absolute label values in case 3,4 in \cref{fig:pairwiserank} increase. In other words, positive pairs will have smaller differences and negative pairs will have larger differences.

In the ablation study, we follow the experiment of \cref{sec:experiment} with values of $0.1, 0.3,0 0.5, 0.7, 0.9$ to systematically explore its effects. The results in \cref{appen-tab:alpha} show that higher $\alpha$ values reduce distinctions between positive pairs, particularly helpful-safe vs. non-helpful-safe, leading to a lower helpfulness score. However, it increases distinctions between negative pairs, especially helpful-unsafe vs. non-helpful-safe, resulting in improved harmlessness, particularly in generative tasks.

\paragraph{The hyperparameter $\tau$.}
The hyperparameter $\tau$ is the coefficient of the KL term in \cref{eq:hs_multiobjectiverl}, which prevents the policy from deviating from the reference policy. In practice, it is important to note that 
 is more related to the training and convergence (as shown in \cref{fig:illustrative}) rather than being a core component to balance the safety and helpfulness.

In our experiments, we follow the settings from \citet{zephyr}, where  $\tau = 0.01$
 is used. This value is applied consistently across all baselines to ensure a fair comparison.
For the ablation study, we adopt values inspired by \citet{gpo}, specifically $\tau = 0.01, 0.1, 0.5, 1.0$. The results  in \cref{appen-tab:tau} indicate that performance can vary significantly with different $\tau$
 values. With different $\tau$, other training hyperparameters, like learning rate, training iterations also need to be carefully chosen.

\paragraph{Ablation on $B_1, B_2, B_3$. }  both  $B_1,  B_3$ must be positive. For this ablation study, we explore the following values $B_3 = 2,1,\frac{1}{2}, \frac{1}{4}$. Given the constraint $B_3(B_1-1) = 1$ as in \cref{eq_appen: constant_equality}, the corresponding values of  $B_1$ are determined for each $B_3$. Additionally, $B_2$ is adjusted to balance the cases described in \cref{fig:pairwiserank} (Case 2 and Case 3). The experiment results are in \cref{tab-appen:b1}.

When $B_3$ is smaller, the label differences for cases 1,2 and 3,4 in \cref{fig:pairwiserank} become less pronounced. For example, in Cases 1 and 2, the pairs (helpful-safe, non-helpful-unsafe) and (helpful-safe, non-helpful-safe) have smaller differences in their label values. This means there is less distinction in whether the non-helpful response is safe or not. As a result, the model shows slightly worse performance in helpfulness but performs better in safety. When $B_3$  is larger, the label differences for the aforementioned two cases become more distinct, and the label value for (helpful-safe, non-helpful-unsafe) becomes significantly higher. This leads the model to prioritize safety more strongly, which results in improved safety performance but a sacrifice in helpfulness.

To conclude, larger $B_3$ values emphasize safety at the expense of helpfulness, while proper values allow for more balanced performance across both objectives.
\begin{table}[h]
    \centering\small
    \caption{Ablation study for $B_1, B_2, B_3$ in \cref{eq:hpgiwl} }
    \label{tab-appen:b1}
    \begin{tabular}{cccccccc}
        \toprule
       $ B_3$ &$ B_1$ & $B_2$ & Values for four cases in \cref{fig:pairwiserank}  & Helpfulness & \multicolumn{3}{c}{Harmlessness}\\
 & & & & Alpaca& Disc.& Gen.&Savg.\\
        \midrule
        2 & 1.5 & -0.25 & 2.5,0.5,-0.5,-0.25 & 9.00 & 58.67 & 95.47 & 77.07 \\
        1 & 2 & -0.5 & 1.5,0.5,-0.5,-1.5 & 11.36 & 60.28 & 95.12 & 77.70 \\
        0.5 & 3 & -1 & 1,0.5,-0.5,-1 & 13.33 & 59.09 & 95.24 & 77.16 \\
        0.25 & 5 & -2 & 0.75, 0.5, -0.5, -0.75 & 13.15 & 59.63 & 94.25 & 76.94 \\
        \bottomrule
    \end{tabular}
\end{table}

\subsection{Full Experiment Results}\label{appen:fullresults}
Here are the details of each open-sourced models:
\begin{itemize}
    \item Zephyr: \url{https://huggingface.co/HuggingFaceH4/zephyr-7b-beta}
    \item Juanako: \url{https://huggingface.co/fblgit/juanako-7b-UNA}
    \item OpenChat: \url{https://huggingface.co/openchat/openchat_3.5}
    \item Mistral: \url{https://huggingface.co/mistralai/Mistral-7B-Instruct-v0.2}
    \item Beaver: \url{https://huggingface.co/PKU-Alignment/beaver-7b-v3.0}
    \item Llama2: \url{https://huggingface.co/meta-llama/Llama-2-7b-chat-hf}
    \item Llama3: \url{https://huggingface.co/meta-llama/Meta-Llama-3-8B-Instruct}
\end{itemize}

\cref{tab:fullresults} shows the full results of all experiments throughout the paper. Here are the details of the data used in our model and the baselines.

We use 4 Nvidia A100 GPUs for each experiment, and the training time for each experiment is around 6 hours for SFT and 6 hours for {\ours}.
For the experiments with red teaming data, we use 1.5K data collected as described in \cref{sec:expredteam} and only performs the {\ours} stage. The training time for this experiment is around 10 minutes with 4 Nvidia A100 GPUs.

\begin{landscape}
\begin{table}[]
\centering\tiny
\caption{Full results of all experiments}
\label{tab:fullresults}
\begin{tabular}{l*{20}{p{0.5cm}}}
\toprule
Model / Method & Alpaca Eval & Crows Pairs & BBQ & Winogrande & Ethics CM & Ethics  Justice & Ethics Deontology & Ethics Utilitarianism & Ethics Virtue & Moral Permissibility & Simple Ethical Questions & Toxigen & HHH Alignment & Real-Toxicity-Prompts & Adv-Bench & ALERT & ALERT Adversarial & Discri-minative Average & Genera-tive Average & Safety Average \\ \midrule
Zephyr-7b-beta & 10.99 & 62.02 & 39.00 & 72.38 & 68.37 & 69.71 & 56.98 & 73.59 & 91.30 & 51.00 & 33.00 & 45.21 & 46.00 & 85.82 & 20.19 & 79.08 & 66.68 & 59.05 & 62.94 & 60.99 \\
Juanako-7b-UNA & 2.88 & 63.74 & 84.00 & 77.43 & 75.96 & 76.41 & 64.10 & 73.79 & 89.13 & 49.00 & 82.00 & 60.96 & 49.00 & 85.90 & 27.50 & 80.70 & 72.79 & 70.46 & 66.72 & 68.59 \\
OpenChatv3.5 & 11.08 & 66.67 & 61.00 & 72.69 & 68.88 & 77.74 & 63.96 & 73.48 & 88.70 & 50.00 & 91.00 & 42.34 & 46.00 & 87.82 & 48.27 & 75.36 & 73.09 & 66.87 & 71.13 & 69.00 \\
Mistral-7B-Instruct-v0.2 & 14.72 & 64.88 & 61.84 & 73.80 & 73.46 & 71.93 & 60.26 & 66.78 & 90.87 & 47.95 & 53.91 & 55.11 & 47.06 & 83.74 & 65.38 & 90.44 & 77.71 & 63.99 & 79.32 & 71.65 \\
Beaver3 & 1.00 & 56.23 & 31.37 & 65.35 & 59.43 & 64.61 & 61.48 & 56.01 & 61.61 & 47.66 & 45.22 & 36.17 & 43.44 & 85.07 & 93.20 & 91.83 & 94.80 & 52.38 & 91.23 & 71.80 \\
Llama2 & 7.60 & 63.98 & 32.99 & 66.46 & 56.14 & 50.00 & 50.00 & 57.97 & 72.00 & 47.37 & 24.35 & 51.00 & 44.34 & 87.91 & 100.00 & 98.62 & 98.32 & 51.38 & 96.21 & 73.80 \\
Llama3 & 22.90 & 63.45 & 60.68 & 71.82 & 58.64 & 70.38 & 64.49 & 62.92 & 81.49 & 48.54 & 54.78 & 45.74 & 45.25 & 89.49 & 99.42 & 95.18 & 95.08 & 60.68 & 94.79 & 77.74 \\\midrule
Mistral + DPO + Helpful Data & 10.99 & 62.02 & 39.00 & 72.38 & 68.37 & 69.71 & 56.98 & 73.59 & 91.30 & 51.00 & 33.00 & 45.21 & 46.00 & 85.82 & 20.19 & 79.08 & 66.68 & 59.05 & 62.94 & 60.99 \\
Mistral + DPO + Safety Data & 4.34 & 65.65 & 39.50 & 74.03 & 64.22 & 55.29 & 50.86 & 60.00 & 89.73 & 46.78 & 38.26 & 47.45 & 45.25 & 87.74 & 100.00 & 99.91 & 99.98 & 56.42 & 96.91 & 76.66 \\
Mistral + DPO + Naive mix Data & 14.71 & 65.59 & 43.68 & 74.27 & 56.47 & 71.01 & 58.20 & 57.15 & 86.71 & 51.17 & 39.13 & 51.06 & 45.70 & 82.49 & 4.23 & 38.64 & 33.46 & 58.35 & 39.71 & 49.03 \\
Mistral + IPO + Naive mix Data & 13.16 & 66.25 & 42.44 & 74.66 & 62.03 & 66.35 & 54.67 & 67.03 & 89.17 & 47.37 & 37.39 & 48.72 & 44.80 & 86.41 & 88.65 & 96.00 & 88.00 & 58.41 & 89.76 & 74.09 \\
Mistral + MORL + Naive mix Data & 10.83 & 61.66 & 39.43 & 71.51 & 68.01 & 67.71 & 55.70 & 72.57 & 91.08 & 50.58 & 33.91 & 44.15 & 46.15 & 87.07 & 21.15 & 82.13 & 69.16 & 58.54 & 64.88 & 61.71 \\
Mistral + BFPO + mix Data & 13.33 & 65.77 & 45.25 & 74.98 & 65.25 & 59.13 & 51.97 & 70.36 & 90.41 & 47.37 & 39.13 & 54.15 & 45.25 & 87.32 & 98.65 & 98.56 & 96.42 & 59.09 & 95.24 & 77.16 \\\midrule
Red Teaming + DPO & 13.07 & 61.84 & 38.95 & 72.45 & 67.77 & 69.12 & 57.48 & 73.63 & 91.42 & 50.88 & 36.52 & 45.11 & 46.15 & 87.99 & 44.00 & 89.22 & 76.33 & 59.28 & 74.39 & 66.83 \\
Red Teaming + IPO & 13.74 & 61.96 & 38.89 & 72.77 & 68.03 & 69.30 & 57.62 & 73.54 & 91.48 & 50.58 & 36.52 & 45.43 & 45.70 & 87.91 & 41.35 & 87.48 & 74.55 & 59.32 & 72.82 & 66.07 \\
Red Teaming + MORL & 12.56 & 61.66 & 38.47 & 71.98 & 66.07 & 69.16 & 56.98 & 73.02 & 91.36 & 51.17 & 33.04 & 44.26 & 45.70 & 87.66 & 21.15 & 82.13 & 69.16 & 58.57 & 65.02 & 61.80 \\
Red Teaming + BFPO & 14.41 & 61.72 & 39.44 & 72.45 & 67.28 & 68.01 & 57.20 & 73.46 & 91.54 & 49.12 & 37.39 & 45.32 & 45.25 & 87.82 & 86.54 & 86.34 & 94.47 & 59.02 & 88.79 & 73.90 \\\midrule
BFPO w\textbackslash{}o shift & 12.76 & 65.95 & 44.44 & 74.98 & 62.50 & 59.87 & 52.28 & 66.64 & 88.78 & 47.66 & 50.43 & 49.89 & 45.70 & 84.32 & 96.15 & 96.90 & 94.12 & 59.09 & 92.87 & 75.98 \\
BFPO w\textbackslash{}o buffer & 15.59 & 65.65 & 44.43 & 74.43 & 61.13 & 69.05 & 56.56 & 67.08 & 89.79 & 47.37 & 46.96 & 53.94 & 45.25 & 84.82 & 85.77 & 95.37 & 89.08 & 60.14 & 88.76 & 74.45 \\\midrule
IPO + LoRA Finetuning & 6.14 & 65.95 & 43.13 & 75.06 & 63.81 & 58.91 & 52.25 & 68.78 & 88.76 & 47.37 & 37.39 & 49.89 & 45.25 & 84.40 & 98.65 & 97.76 & 95.07 & 58.05 & 93.97 & 76.01 \\
BFPO + LoRA Finetuning & 7.77 & 66.25 & 46.41 & 74.90 & 72.20 & 73.00 & 58.40 & 66.53 & 91.38 & 47.08 & 64.35 & 66.06 & 45.70 & 85.82 & 98.85 & 98.23 & 96.03 & 64.36 & 94.73 & 79.54 \\\midrule
BFPO $\alpha=0.1$ & 13.61 & 65.59 & 42.91 & 74.27 & 62.52 & 68.16 & 55.73 & 66.87 & 90.29 & 47.95 & 44.35 & 53.40 & 45.70 & 83.49 & 82.31 & 94.96 & 88.81 & 59.81 & 87.39 & 73.60 \\
BFPO $\alpha=0.3$ & 14.06 & 65.41 & 43.71 & 74.82 & 62.50 & 67.01 & 55.76 & 68.34 & 90.19 & 47.66 & 50.43 & 52.13 & 45.70 & 84.82 & 93.46 & 96.61 & 92.05 & 60.31 & 91.73 & 76.02 \\
BFPO $\alpha=0.5$ & 13.33 & 65.77 & 45.25 & 74.98 & 65.25 & 59.13 & 51.97 & 70.36 & 90.41 & 47.37 & 39.13 & 54.15 & 45.25 & 87.32 & 98.65 & 98.56 & 96.42 & 59.09 & 95.24 & 77.16 \\
BFPO $\alpha=0.7$ & 9.01 & 65.71 & 43.06 & 75.06 & 68.34 & 54.33 & 51.22 & 67.74 & 89.97 & 47.08 & 26.96 & 53.83 & 44.80 & 86.57 & 99.81 & 99.72 & 99.01 & 57.34 & 96.28 & 76.81 \\
BFPO $\alpha=0.9$ & 7.21 & 65.83 & 44.44 & 75.30 & 69.03 & 51.48 & 50.44 & 68.41 & 89.25 & 47.08 & 21.74 & 49.47 & 44.80 & 87.49 & 99.81 & 99.84 & 99.50 & 56.44 & 96.66 & 76.55 \\\midrule
BFPO $\tau=0.01$ & 13.33 & 65.77 & 45.25 & 74.98 & 65.25 & 59.13 & 51.97 & 70.36 & 90.41 & 47.37 & 39.13 & 54.15 & 45.25 & 87.32 & 98.65 & 98.56 & 96.42 & 59.09 & 95.24 & 77.16 \\
BFPO $\tau=0.1$ & 6.4 & 66.01 & 40.64 & 74.82 & 63.78 & 57.10 & 51.78 & 63.66 & 88.50 & 47.37 & 22.61 & 45.11 & 43.89 & 84.82 & 63.65 & 91.01 & 86.31 & 55.44 & 81.45 & 68.44 \\
BFPO $\tau=0.5$ & 6.53 & 66.31 & 38.78 & 74.74 & 62.08 & 54.29 & 51.17 & 60.69 & 86.37 & 48.54 & 17.39 & 43.40 & 44.34 & 84.90 & 50.00 & 91.54 & 86.10 & 54.01 & 78.14 & 66.07 \\
BFPO $\tau=1.0$ & 6.74 & 66.01 & 39.34 & 74.90 & 62.08 & 55.36 & 51.64 & 60.17 & 85.87 & 49.12 & 17.39 & 43.40 & 43.89 & 85.07 & 48.08 & 91.04 & 85.91 & 54.10 & 77.52 & 65.81 \\\midrule
BFPO $B_3=2$ & 9.00 & 65.65 & 45.72 & 75.61 & 70.35 & 52.59 & 50.61 & 68.51 & 89.73 & 46.78 & 38.26 & 54.57 & 45.70 & 85.49 & 99.62 & 98.92 & 97.85 & 58.67 & 95.47 & 77.07 \\
BFPO $B_3=1$ & 11.36 & 65.65 & 44.00 & 74.51 & 65.97 & 65.27 & 54.51 & 70.61 & 90.39 & 47.37 & 46.09 & 53.72 & 45.25 & 85.65 & 99.62 & 98.93 & 96.30 & 60.28 & 95.12 & 77.70 \\
BFPO $B_3=0.5$ & 13.33 & 65.77 & 45.25 & 74.98 & 65.25 & 59.13 & 51.97 & 70.36 & 90.41 & 47.37 & 39.13 & 54.15 & 45.25 & 87.32 & 98.65 & 98.56 & 96.42 & 59.09 & 95.24 & 77.16 \\
BFPO $B_3=0.25$ & 13.15 & 65.59 & 43.96 & 75.06 & 62.88 & 65.46 & 54.59 & 68.64 & 89.83 & 47.66 & 43.48 & 53.19 & 45.25 & 86.07 & 98.46 & 97.85 & 94.62 & 59.63 & 94.25 & 76.94 \\ \bottomrule
\end{tabular}
\end{table}
\end{landscape}

\end{document}